\documentclass{article}

\usepackage{arxiv}

\usepackage{latexsym,amsmath,amssymb,amscd,bm,amsfonts}
\usepackage{algorithm}
\usepackage[noend]{algpseudocode}
\usepackage{graphicx}
\usepackage{textcomp}
\usepackage{xcolor}
\usepackage{blindtext}
\usepackage[hidelinks]{hyperref}
\usepackage{subcaption}
\usepackage{multirow}
\usepackage{booktabs}
\usepackage{amsthm}

\usepackage{makecell}
\usepackage{tikz}
\usetikzlibrary{positioning, shapes, arrows.meta, fit, backgrounds, calc}

\newtheorem{thm}{Theorem}[section]
\newtheorem{lem}[thm]{Lemma}

\theoremstyle{definition}

\newtheorem{rem}{Remark}

\newcommand*{\mylemmaref}[2][Lemma~]{\hyperref[{#2}]{#1\ref*{#2}}}
\newcommand*{\mycorref}[2][Corollary~]{\hyperref[{#2}]{#1\ref*{#2}}}
\newcommand*{\myeqref}[2][Equation~]{\hyperref[{#2}]{#1(\ref*{#2})}}
\def\equationautorefname#1#2\null{Equation#1(#2\null)}

\newcommand{\bestvalue}[1]{\textcolor{blue}{\textbf{#1}}}

\title{Layer-Specific Lipschitz Modulation for Fault-Tolerant Multimodal Representation Learning}

\author{ \href{https://orcid.org/0009-0005-7928-5874}{\includegraphics[scale=0.06]{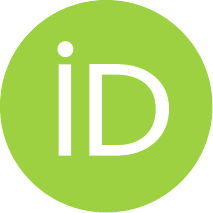}\hspace{1mm}Diyar Altinses} and \href{https://orcid.org/0000-0001-8405-0977}{\includegraphics[scale=0.06]{orcid.pdf}\hspace{1mm}Andreas Schwung}\thanks{This Article is funded by the Open Access Publication Fund of South Westphalia University of Applied Sciences.\\
Submitted to \textit{Information Fusion}} \\
	Department of Automation Technology and Learning Systems\\
	South Westphalia University of Applied Sciences\\
	Soest, Germany \\
	\texttt{altinses.diyar@fh-swf.de},  \texttt{schwung.andreas@fh-swf.de}\\
}

\renewcommand{\shorttitle}{Layer-Specific Lipschitz Modulation}

\hypersetup{
pdftitle={Layer-Specific Lipschitz Modulation for Fault-Tolerant Multimodal Representation Learning},
pdfsubject={cs.AI, cs.LG},
pdfauthor={Diyar Altinses, Andreas Schwung},
pdfkeywords={Fault-tolerant multimodal learning, self-supervised representation learning, Lipschitz regularization, anomaly detection and correction, robustness analysis},
}

\begin{document}
\maketitle

\begin{abstract}
	Modern multimodal systems deployed in industrial and safety-critical environments must remain reliable under partial sensor failures, signal degradation, or cross-modal inconsistencies. This work introduces a mathematically grounded framework for fault-tolerant multimodal representation learning that unifies self-supervised anomaly detection and error correction within a single architecture. Building upon a theoretical analysis of perturbation propagation, we derive Lipschitz- and Jacobian-based criteria that determine whether a neural operator amplifies or attenuates localized faults. Guided by this theory, we propose a two-stage self-supervised training scheme: pre-training a multimodal convolutional autoencoder on clean data to preserve localized anomaly signals in the latent space, and expanding it with a learnable compute block composed of dense layers for correction and contrastive objectives for anomaly identification. Furthermore, we introduce layer-specific Lipschitz modulation and gradient clipping as principled mechanisms to control sensitivity across detection and correction modules. Experimental results on multimodal fault datasets demonstrate that the proposed approach improves both anomaly detection accuracy and reconstruction under sensor corruption. Overall, this framework bridges the gap between analytical robustness guarantees and practical fault-tolerant multimodal learning.
\end{abstract}

\keywords{Fault-tolerant multimodal learning \and self-supervised representation learning \and Lipschitz regularization \and anomaly detection and correction \and robustness analysis}

\section{Introduction}

Unplanned downtime represents a critical failure mode and a major source of operational inefficiency across industrial sectors \cite{rahman2014assessment, diez2019data}. This unscheduled termination of production leads to substantial, multi-faceted economic losses, consistently ranking as a primary financial burden. These costs are derived fundamentally from three main categories: lost production volume, which directly impacts revenue, the significant expense of emergency maintenance and specialized repairs, and harmful effects on supply chain stability and contractual obligations \cite{wang2024survey}. 

Traditional strategies for mitigating system failure have historically relied heavily on hardware redundancy. This approach, while effective against component-specific malfunctions like random failures or wear-out, offers limited protection against common-cause failures \cite{alizadeh2017reliability}. These failures originate from shared environmental effects, such as excessive heat, high humidity, or mechanical vibration \cite{jones2025common}. When these conditions impact multiple components simultaneously, even redundant ones, the safety margin is eliminated \cite{altinses2023deep}. 

A modern approach uses multimodal machine learning to overcome the limitations of traditional, single-sensor diagnostic methods and redundancy. Machine learning-based methodologies enable the capability to fuse diverse, heterogeneous sensor data, sourced from spatially distributed and often disparate information streams, into one general model \cite{ali2025multimodal}. This fusion process integrates variables such as temporal sensor readings and camera data, resulting in a richer and more accurate estimation of the current state. Crucially, this advanced diagnostic capability enables compensating for the weakness of one modality with the strength of another. This creates an inherently fault-tolerant and adaptive production system \cite{altinses2025enhancing}.

A complete fault tolerance mechanism fundamentally requires a sequential three-stage process: fault injection, detection, and correction (or recovery). Existing research and industrial implementations often exhibit a fragmented focus, addressing these stages individually rather than as an integrated system \cite{zhang2025distributed, altinses2025fault}. While some state-of-the-art approaches attempt to incorporate all three steps, they frequently encounter a critical performance trade-off: improvements in the correction mechanism often lead to a degradation in the detection stage \cite{altinses2025fault}. 

Building upon the identified limitations, this research presents a theoretical analysis of how various perturbations impact different neural network layer structures. We extend this analysis by investigating the dynamic behavior of neural network sensitivity during the optimization. Informed by these theoretical insights, we propose a self-supervised, dual-regularized deep learning architecture. This novel approach is specifically designed to address the detection-correction trade-off, aiming to increase network sensitivity for robust anomaly detection while simultaneously reducing sensitivity for a stable and accurate correction process. This integrated framework represents a significant advancement toward highly efficient industrial fault tolerance. The contributions of this paper are summarized as follows:

\begin{enumerate}
    \item We present a theoretical framework analyzing the effect of diverse operational perturbations on the structural and dynamic sensitivity of deep neural network architectures.
    \item We developed, based on theory, a dual regularization mechanism that selectively increases sensitivity for anomaly detection tasks while simultaneously decreasing it for the fault correction process, effectively resolving the detection-correction performance trade-off.
    \item We validated the proposed approach on three multimodal heterogeneous industrial datasets, demonstrating a unified fault-tolerance system that achieves superior performance compared to existing fragmented methodologies.
\end{enumerate}

The remainder of this paper is structured to systematically present our research and findings. \autoref{sec:rel} provides a detailed overview of the related work, focusing on advances in multimodal self-supervised learning and existing fault-tolerant deep learning methodologies. Following this, \autoref{sec:pertub} establishes the foundational theoretical analysis concerning perturbation effects and the dynamic sensitivity of neural networks. The practical application of these theoretical insights is presented in \autoref{sec:method}, which introduces and fully details our proposed self-supervised, dual-path, and dual-regularized model. \autoref{sec:eval} then validates our approach through a comprehensive experimental evaluation and discussion of results. Finally, \autoref{sec:conclusion} summarizes the key findings of this work and offers concluding remarks.


\section{Related Work}
\label{sec:rel}

Research on robust multimodal representation learning and fault-tolerant neural architectures has evolved along several directions. Three particularly relevant fields related to our approach are: multimodal self-supervised learning, fault-tolerant multimodal learning, and Lipschitz-based sensitivity regulation. The following subsections review the most pertinent literature within each domain.

\subsection{Multimodal Self-Supervised Learning}

Multimodal self-supervised learning has matured rapidly, driven by architectures and objectives that exploit cross-view (or cross-modal) correspondences to learn rich joint representations without manual labels. Early audio-visual methods established the utility of cross-modal signals: Arandjelović and Zisserman showed that audio-visual correspondence yields strong cross-modal embeddings and localization cues \cite{arandjelovic2017look} and later extended this to object-sound grounding \cite{arandjelovic2018objects}. Temporal synchronization and audio–visual temporal self-supervision were shown to be effective for learning temporally aligned features from video \cite{korbar2018cooperative,morgado2020learning}.

Contrastive and multiview objectives generalized these ideas. Tian et al. propose the contrastive multiview coding, which includes a principled noise-contrastive objective for arbitrary sets of views and demonstrated strong transfer performance \cite{tian2020contrastive}. Deep multimodal frameworks have extended contrastive training to multiple modalities in video, showing that versatile multimodal representations can be learned end-to-end \cite{xiang2025integrating}. Transformer-based multimodal pretraining further scaled self-supervision to raw audio, video, and text with contrastive/fusion objectives \cite{akbari2021vatt}\cite{chen2025pretraining}.

In vision–language modeling, researchers demonstrate that two-stream and fusion transformers with large-scale pretraining, masked modeling, and matching objectives yield powerful joint embeddings for downstream tasks \cite{lu2019vilbert,tan2019lxmert,din2025multimodal}. More recently, Radford et al. showed that large contrastive language–image models, that pair images with natural-language descriptions, provide a scalable source of multimodal supervision that transfers broadly \cite{radford2021learning}.

These lines of work provide the foundations for our approach. However, the literature largely focuses on representation quality and transfer. Comparatively little work studies fault-tolerance and latent correction mechanisms in a principled, Lipschitz-aware manner. We close this gap by proposing a two-stage, self-supervised pipeline that combines contrastive pretraining (to preserve anomaly salience) with latent-space correction under Lipschitz-controlled regularization to enable both detection and correction.

\subsection{Fault-tolerant Multimodal Learning}

Fault tolerance in multimodal learning sits at the intersection of robust fusion, missing-modality inference, and anomaly/resilience mechanisms. Several recent works have studied the robustness properties of multimodal models and proposed architectures and training schemes to mitigate the effects of partial or corrupted modalities. Ma et al. conduct a systematic study of Transformer-based multimodal models and show that fusion strategy and model design critically affect robustness to missing modalities \cite{ma2022multimodal}. Complementary approaches explicitly model missing modalities via imagination or imputation networks that predict absent modality embeddings from available views \cite{zhao2021missing}.

Several recent papers propose principled strategies for robust fusion. Representation decoupling reduces interference across modality combinations, thereby improving robustness under incomplete inputs \cite{wei2024robust}. Wang et al. explicitly separate common and modality-unique latent factors to enhance resilience to missing or corrupted inputs \cite{wang2024decoupling}. Other works design modality reliability scoring or uncertainty-aware fusion to down-weight unreliable streams and improve end-to-end robustness \cite{gao2024embracing, hong2023watch}.

A growing literature focuses on the problem of incomplete multimodal learning under uncertain missingness, offering distributional and generative remedies to recover missing modalities or construct robust joint representations \cite{lan2024robust, li2024toward}. Contrastive and sampling-based strategies have also been adapted to create modality-agnostic latent spaces that are less sensitive to perturbations \cite{wei2024robust, lin2023missmodal}.

A significant portion of recent literature focuses on the performance and robustness of multimodal autoencoders. One technique is the fuzzy regularization approach, which employs a specialized regularizer to amplify the magnitudes of correct features while actively attenuating corrupted ones \cite{altinses2023deep}. Further elaborating on the latent space, researchers in \cite{altinses2025enhancing} proposed extending the multimodal autoencoders to a variational version, implementing a probabilistic fusion mechanism to better model uncertainty and variability during the combination of latent representations. Concurrently, other work has concentrated on dynamic feature selection, exemplified by the concept fusion developed in \cite{altinses2025fault}. This block enables the network to learn selective integration, dynamically choosing between joint and marginal information from the modalities.

Despite progress, most methods evaluate robustness empirically. Few provide principled, layerwise, or Lipschitz-aware design rules that jointly address anomaly localization and correction in latent space. Our work fills this gap by linking convolutional locality to preserved perturbation salience, proposing linear compute blocks to disperse and correct latent-encoded faults, and introducing a layer-specific Lipschitz regularization strategy that explicitly trades off detectability and correctability.

\subsection{Lipschitz-based Sensitivity Regulation}

Research on regulating model sensitivity via Lipschitz continuity and Jacobian control has grown rapidly, driven by concerns of robustness, generalization, and responsible model behavior. Early work on Lipschitz-constrained networks showed that bounding the operator norm of weight matrices improves adversarial robustness and calibration \cite{anil2019sorting}. Miyato et al. introduced Jacobian regularization for semi-supervised learning to improve smoothness and generalization \cite{miyato2018virtual}.  
Spectral normalization appeared as a practical mechanism to bound layer Lipschitz constants in GANs and classification networks \cite{miyato2018spectral}. Fazlyab et al. analysed contraction-based stability in feedback systems via Lipschitz bounds and offered semidefinite relaxations for certified robustness \cite{fazlyab2023certified}.

More recently, Jacobian‐based penalties have been applied to large-scale pretraining: Chen et al. demonstrated that minimizing the norm of the Jacobian improves out-of-distribution robustness in massive transformer models \cite{chen2023understanding}. On the generative modelling front, Rhodes et al. proposed controlling the Lipschitz to balance sample quality and latent stability \cite{rhodes2021local}. From a theoretical standpoint, Lewandowski et al. derived generalization bounds in deep learning via bounds on the spectral norm of the Jacobian of the network mapping \cite{lewandowski2024learning}. Moreover, Chen et al. established a convergence guarantee for gradient-clipped SGD by relating clipping thresholds to implicit Lipschitz constraints on the network and loss \cite{chen2020understanding}. Finally, Altinses et al. provided a theoretical analysis of the Lipschitz continuity associated with standard fusion strategies in multimodal autoencoders, highlighting the stability limitations of fusion methods. Building on these gradient dynamics, they proposed a regularized attention-based fusion mechanism that explicitly constrains the Lipschitz constant to ensure robust training stability in industrial applications \cite{altinses2025stabilizing}.

These works provide the foundation for our layer-specific Lipschitz modulation strategy: we adapt Lipschitz regularization not uniformly but selectively across layers with distinct roles (encoder vs. compute block) to simultaneously support anomaly localization via higher sensitivity and error correction via contraction.


\section{Propagation of Localized Perturbations}
\label{sec:pertub}

In multimodal learning architectures, distinct modalities such as image data and sensor data are often processed via different neural network components. Typically, convolutional layers are employed for image modalities to leverage their inherent spatial structure, whereas fully connected (dense) layers are used for sensor modalities due to their tabular nature. Understanding how localized perturbations, such as faults or anomalies injected into the input, propagate through these differing architectures is critical for interpreting model robustness and the effectiveness of fault detection mechanisms. Intuitively, convolutional layers tend to preserve spatial locality by restricting the impact of perturbations to a small subset of output features, whereas dense layers distribute perturbations across all output dimensions, potentially diluting localized signals. In this section, we rigorously formalize and prove this intuition for the case of multiple corrupted input entries. 

\subsection{Additive Perturbation Model}

Additive perturbations are the most fundamental and widely used abstraction for modeling noise and local faults in both theoretical analysis and engineering practice. They correspond to the intuitive notion that a clean signal $x$ is affected by an independent disturbance $\delta$, which may arise from sensor errors, environmental noise, or transmission corruption. 
In this setting, the observed input is written as
\begin{align}
    \tilde{x} = x + \delta,
\end{align}
where the perturbation vector $\delta$ has support $\operatorname{supp}(\delta) = S\subseteq\{1,\dots,N\}$, representing the (possibly small) subset of coordinates where the fault actually occurs.

Having established the additive perturbation framework, we now examine how such perturbations propagate through the two most common linear mappings in deep learning: fully connected (dense) layers and convolutional layers. Dense layers form the backbone of many sensor-processing or tabular-data subsystems and are characterized by their global receptive field: each output neuron aggregates contributions from every input coordinate. Consequently, any local perturbation in the input potentially affects all output dimensions. Our goal is to make this intuition mathematically explicit by quantifying the expected perturbation energy at the layer output under random weights. Therefore, we begin with the dense case.

\begin{lem}
\label{lem:dense_add}
    Let $W\in\mathbb{R}^{M\times N}$ have independent entries $W_{j,i}$ with $\mathbb{E}[W_{j,i}]=0$ and $\operatorname{Var}[W_{j,i}]=\sigma_W^2$. Let $\delta$ be supported on $S$. Then the expected squared output perturbation satisfies
    \[
    \mathbb{E}\big[\|W\delta\|_2^2\big] \;=\; M\,\sigma_W^2\,\|\delta_S\|_2^2.
    \]
\end{lem}
\begin{proof}
    Write
    \[
    \|W\delta\|_2^2 = \delta^\top W^\top W \delta = \sum_{i,i'} \delta_i \delta_{i'} (W^\top W)_{i,i'}.
    \]
    Take expectations. For the matrix $W^\top W$ the entries are
    \[
    (W^\top W)_{i,i'} = \sum_{j=1}^M W_{j,i} W_{j,i'}.
    \]
    Thus
    \[
    \mathbb{E}[(W^\top W)_{i,i'}] = \sum_{j=1}^M \mathbb{E}[W_{j,i} W_{j,i'}].
    \]
    Because entries in distinct columns are independent and zero-mean, for $i\neq i'$ we have $\mathbb{E}[W_{j,i}W_{j,i'}]=\mathbb{E}[W_{j,i}]\mathbb{E}[W_{j,i'}]=0$. For $i=i'$,
    $\mathbb{E}[W_{j,i}^2] = \operatorname{Var}[W_{j,i}] = \sigma_W^2$. Hence
    \[
    \mathbb{E}[W^\top W] = M\sigma_W^2 I_N.
    \]
    Therefore
    \[
    \mathbb{E}\|W\delta\|_2^2 = \delta^\top \mathbb{E}[W^\top W]\delta = \delta^\top (M\sigma_W^2 I)\delta = M\sigma_W^2 \|\delta\|_2^2,
    \]
    and restricting to the support $S$ yields the stated formula.
\end{proof}

This result shows that each input perturbation component contributes equally (in expectation) to every output coordinate. The total perturbation energy grows linearly with the number of outputs $M$, but the per-output expected energy equals $\sigma_W^2\|\delta_S\|_2^2$, which is independent of $M$. Thus, a localized input perturbation appears as many small perturbations across outputs, a global spread which reduces per-coordinate salience. Note that this analysis primarily applies to the initial state (He/Xavier Initialization) or under the assumption of centered data/weights, which is often approximated by batch normalization.

While \mylemmaref{lem:dense_add} reveals that dense layers act as global mixing operators that eliminate the spatial structure of additive perturbations, convolutional layers behave fundamentally differently. Their defining property is weight sharing combined with local receptive fields, meaning that each output neuron only depends on a small neighborhood of input coordinates. This architectural locality strongly constrains how far a perturbation can spread through the layer. Therefore, to contrast with the global propagation observed in the dense case, we now analyze the expected perturbation energy in a convolutional layer under the same additive perturbation model. Therefore, we model a convolutional layer $C\in\mathbb{R}^{M'\times N}$ as a sparse matrix where each column $C_{\cdot, i}$ has nonzero entries only at output indices in $S_i$ (the receptive field positions for input index $i$); assume $|S_i|\le K$. Nonzero entries (kernel positions) are independent, zero-mean, and have the variance $\sigma_k^2$.

\begin{lem}
\label{lem:conv_add}
    Under the stated model,
    \[
    \mathbb{E}\big[\|C\delta\|_2^2\big] \;=\; \sigma_k^2 \sum_{i\in S} |S_i|\,\delta_i^2
    \;\le\; \sigma_k^2\, rK \,\|\delta_S\|_2^2,
    \]
    where $r=|S|$ and $s:=\sum_{i\in S}|S_i|\le rK$ is the multiplicity-corrected number of affected outputs.
\end{lem}

\begin{proof}
    As before,
    \[
    \|C\delta\|_2^2 = \delta^\top C^\top C \delta = \sum_{i,i'} \delta_i \delta_{i'} (C^\top C)_{i,i'},
    \]
    and
    \[
    (C^\top C)_{i,i'} = \sum_{j=1}^{M'} C_{j,i} C_{j,i'}.
    \]
    For fixed $j$, independence and zero-mean of the nonzero kernel entries imply $\mathbb{E}[C_{j,i} C_{j,i'}]=0$ when $i\neq i'$ (unless by model design both entries correspond to the same kernel parameter, which we excluded by independence). For $i=i'$ and $j\in S_i$, $\mathbb{E}[C_{j,i}^2]=\sigma_k^2$; for $j\not\in S_i$, $C_{j,i}=0$. Therefore
    \[
    \mathbb{E}[(C^\top C)_{i,i'}] = \begin{cases}
    \sigma_k^2 |S_i|, & i=i',\\
    0, & i\ne i'.
    \end{cases}
    \]
    Hence
    \[
    \mathbb{E}\|C\delta\|_2^2 = \sum_{i\in S} \delta_i^2 \sigma_k^2 |S_i|,
    \]
    \[
    \sum_{i\in S} \delta_i^2 \sigma_k^2 |S_i| \le \sigma_k^2 \big(\max_{i\in S}|S_i|\big) \sum_{i\in S} \delta_i^2 \le \sigma_k^2 rK \|\delta_S\|_2^2,
    \]
    which proves the lemma.
\end{proof}

Unlike the dense case, each perturbed input coordinate $i$ affects only $|S_i|$ outputs. Thus, the total affected output count $s$ scales with $rK$ and may be much smaller than the dense output dimension $M$. The convolution preserves spatial locality and keeps the perturbation energy concentrated into a bounded neighborhood instead of being globally spread.

\subsection{Multiplicative perturbations}

While the additive model provides a mathematically convenient and widely applicable framework for studying noise propagation, it assumes that perturbations are independent of the underlying signal. In many real-world systems, however, this assumption does not hold: the magnitude or direction of a disturbance often depends on the local signal amplitude itself. Typical examples include gain fluctuations in sensors, feature-dependent dropout, fading in communication channels, and multiplicative calibration errors in imaging systems. To capture such effects, we now turn to the multiplicative perturbation model, which introduces a signal-dependent scaling of each input dimension.

Formally, we represent the perturbed input as
\begin{align}
    \tilde{x} = x \odot (1 + \epsilon),
\end{align}
where \(\odot\) denotes the elementwise (Hadamard) product and \(\epsilon \in \mathbb{R}^N\) is a random or deterministic perturbation vector supported on a subset \(S \subseteq \{1, \dots, N\}\). For linear layers, this is precisely an additive perturbation with $\delta = \operatorname{diag}(x)\varepsilon$. Equivalently, we can express the perturbation as
\begin{align}
    \tilde{x} = x + (x \odot \epsilon) = x + \delta,
\end{align}
so that the effective additive perturbation is \(\delta_i = x_i \epsilon_i\). This formulation highlights that the multiplicative model generalizes the additive one: the additive case corresponds to \(\epsilon_i = \delta_i / x_i\) for nonzero \(x_i\), i.e., perturbations that are independent of the input magnitude.

\begin{lem}
\label{lem:dense_mul}
    Let $W$ as in Lemma \ref{lem:dense_add} and let $\varepsilon$ be supported on $S$. Then
    \[
    \mathbb{E}\big[\|W(\operatorname{diag}(x)\varepsilon)\|_2^2\big]
    = M\,\sigma_W^2 \sum_{i\in S} x_i^2 \varepsilon_i^2.
    \]
\end{lem}

\begin{proof}
    Set $\delta := \operatorname{diag}(x)\varepsilon$, i.e.\ $\delta_i = x_i \varepsilon_i$. Apply \mylemmaref{lem:dense_add} to $\delta$. The result follows directly.
\end{proof}

\begin{lem}
\label{lem:conv_mul}
    Let $C$ as in Lemma \ref{lem:conv_add} and $\varepsilon$ supported on $S$. Then
    \[
    \mathbb{E}\big[\|C(\operatorname{diag}(x)\varepsilon)\|_2^2\big]
    = \sigma_k^2 \sum_{i\in S} |S_i|\, x_i^2 \varepsilon_i^2
    \le \sigma_k^2\, rK \sum_{i\in S} x_i^2 \varepsilon_i^2.
    \]
\end{lem}

\begin{proof}
    Same substitution: put $\delta = \operatorname{diag}(x)\varepsilon$ and apply \mylemmaref{lem:conv_add}.
\end{proof}

Multiplicative perturbations scale with the local signal amplitude $x_i$: areas with larger $x_i$ amplify relative errors $\varepsilon_i$ into bigger absolute distortions. Structurally, convolution still confines these amplified distortions to small neighborhoods, whereas dense layers spread the signal-dependent distortion across all outputs.

\subsection{Comparison of Propagation Dynamics}

Having established the separate propagation characteristics of additive and multiplicative perturbations through dense and convolutional layers, we now synthesize these results into a unified comparative analysis. While the individual lemmas quantified the expected perturbation energy in isolation, a direct comparison reveals the fundamental structural difference between global and local mappings. The following theorem formalizes this distinction by comparing the expected perturbation energies in both architectures under equivalent statistical assumptions. The subsequent corollary interprets these findings in terms of architectural design principles, linking the theoretical analysis to practical implications for anomaly detection and fault correction.

\begin{thm}
\label{thm:comparison}
    Under the assumptions of \mylemmaref{lem:dense_add} and \mylemmaref{lem:conv_add} and their multiplicative analogues, the expected squared output energies satisfy
    \[
    \mathbb{E}\big[\|W\delta\|_2^2\big] = M\sigma_W^2 \|\delta_S\|_2^2,
    \qquad
    \mathbb{E}\big[\|C\delta\|_2^2\big] = \sigma_k^2 \sum_{i\in S} |S_i| \delta_i^2.
    \]
    Define $s = \sum_{i\in S}|S_i|\le rK$. If $\sigma_W^2\approx\sigma_k^2$ and $s\ll M$, then the convolutional output concentrates the perturbation energy into far fewer coordinates than the dense output. Consequently, the average per-affected-output expected energy is larger for the convolutional mapping:
    \[
    \frac{\mathbb{E}\|C\delta\|_2^2}{s} \ge \frac{\mathbb{E}\|W\delta\|_2^2}{M},
    \]
    under $\sigma_k^2=\sigma_W^2$ and if the $\delta_i^2$ are distributed appropriately.
\end{thm}

\begin{proof}
    The equalities are restatements of \mylemmaref{lem:dense_add} and \mylemmaref{lem:conv_add}. Dividing both sides by the number of affected outputs yields the per-output averages. If $\sigma_k^2=\sigma_W^2$ and $s\ll M$, then the denominator on the left is much smaller, so the typical per-output energy is larger for the convolutional mapping.
\end{proof}

Under the conditions of \autoref{thm:comparison}, suppose $s \ll M$ and the variances $\sigma_W^2$ and $\sigma_k^2$ are of comparable order. Then, for a fixed perturbation energy $\|\delta_S\|_2^2$, the convolutional layer concentrates this energy into a significantly smaller number of output coordinates than the dense layer. This concentration implies that, on average, the perturbation amplitude per affected output coordinate is higher for the convolutional layer, thereby increasing the salience of the anomaly signal in latent representations and enhancing its detectability.


\begin{rem}
    The independence and disjoint-support assumptions in \mylemmaref{lem:dense_add} and \mylemmaref{lem:conv_add} are idealized. In practical convolutional networks, receptive fields often overlap, so $S_i \cap S_{i'} \neq \emptyset$ for distinct perturbed indices $i$ and $i'$, and the corresponding weight entries may not be perfectly independent (e.g., due to learned structure after training). Overlapping supports introduce additional cross terms in the expectation, which can either increase or decrease the total output perturbation energy depending on the correlation pattern. Nonetheless, as long as the effective number of affected outputs $s$ remains much smaller than $M$, the qualitative conclusion of Theorem 1, that convolutional mappings localize perturbations relative to dense mappings, still holds. Quantifying the precise effect in the correlated and overlapping case requires extending the analysis to account for the joint covariance structure of the weights and receptive fields.
\end{rem}

This difference in perturbation propagation has practical implications for fault and anomaly detection in multimodal systems. The localized preservation of perturbation magnitude in convolutional networks enables more precise identification and localization of faults in image data, as the perturbations remain distinct and spatially confined in the latent representations. In contrast, dense layers used for 1-dimensional data tend to diffuse such perturbations throughout the latent space, potentially making fault signals less distinguishable. This means that architectural choices, specifically the use of convolutional versus dense layers, play a fundamental role in how faults manifest in latent representations. For effective multimodal fault detection, understanding and leveraging these propagation characteristics can guide the design of network components and training strategies that enhance anomaly sensitivity and robustness.

\subsection{Training Dynamics and Noise-Dependent Lipschitz Sensitivity}
\label{sec:gradclip}

The previous analysis demonstrated that the architectural choice determines the spatial support of perturbation propagation. However, the magnitude of this propagation and the model's overall stability depend not only on the connectivity pattern but also on the global sensitivity of the mapping $f_\theta$. While our structural analysis was conducted in the static case with fixed parameters, in practice, these parameters evolve during optimization. Therefore, it is essential to understand how the training dynamics interact with the model's sensitivity, and how these sensitivities can be optimized to either stabilize learning for correction or amplify discriminative features for anomaly detection. 

The measure of sensitivity is the Lipschitz constant of the network function with respect to its input. Let $f_\theta:\mathbb{R}^n\rightarrow\mathbb{R}^m$ be a neural network parameterized by weights $\theta$. We denote by $J_f(x)=\frac{\partial f_\theta(x)}{\partial x}$ the Jacobian of $f_\theta$ with respect to its input $x$.
The global Lipschitz constant is then defined as
\begin{align}
    L_f = \sup_{x\neq y} \frac{\left\|f_\theta(x)-f_\theta(y)\right\|_2}{\|x-y\|_2},   
\end{align}
which measures the maximal amplification of input perturbations through the network.

The target is to understand how training regimes and input noise modify $L_f$. This provides a quantitative measure of how stable or sensitive the network is, and thus how it may be guided by gradient clipping, Jacobian regularization, or noise-based training strategies. The following \autoref{thm:lipschitz_clean} presents the clean input Lipschitz bound.

\begin{thm}
\label{thm:lipschitz_clean}
    Let $f_\theta(x)$ be an $L$-layer feed-forward network $f_\theta(x) = W_L\phi_{L-1}( W_{L-1}\phi_{L-2}(\dots \phi_1(W_1 x)\dots))$ with weight matrices $W_\ell$ and elementwise activation functions $\phi_\ell$ that are $K_\phi$-Lipschitz continuous. Then, for all $x,y\in\mathbb{R}^n$, it holds that
    \[
        \|f_\theta(x)-f_\theta(y)\|_2 \le \Big(\prod_{\ell=1}^{L} \|W_\ell\|_2 K_\phi\Big)\, \|x-y\|_2.
    \]
    Hence, the global Lipschitz constant satisfies
    \[
        L_f^{(\mathrm{clean})} \le K_\phi^L \prod_{\ell=1}^{L} \|W_\ell\|_2.
    \]
\end{thm}

\begin{proof}
The proof follows directly from the submultiplicativity of operator norms and the chain rule for composition of Lipschitz functions. For each layer $h_\ell(x)=W_\ell \phi_{\ell-1}(\cdot)$, we have
\[
    \|h_\ell(x_1)-h_\ell(x_2)\|_2 \le \|W_\ell\|_2 K_\phi \|x_1-x_2\|_2.
\]
Applying this inequality recursively over $L$ layers yields the stated bound.
\end{proof}

\begin{rem}
    Connection to Architecture: The product $\prod \|W_\ell\|_2$ in \autoref{thm:lipschitz_clean} highlights the synergy with the findings in \autoref{sec:pertub}. While convolutional layers restrict the Jacobian's sparsity structure, keeping the spectral norm $\|W_\ell\|_2$ low ensures that the active paths do not amplify the signal excessively. Thus, robustness is a dual objective: structural localization via architecture and spectral control via training.
\end{rem}

This theorem describes the static sensitivity of the network to clean inputs. The Lipschitz constant grows multiplicatively with layer depth and the spectral norms of the weights. If each layer’s operator norm is tightly controlled, e.g., using spectral normalization or weight decay, the overall sensitivity $L_f$ remains small, implying stable propagation of input differences. Conversely, large layer norms or strong activations lead to high $L_f$, which makes the network highly sensitive to small input changes, which is useful for detection tasks but detrimental to correction stability.

We now consider the case where the input is corrupted by additive noise, $x' = x + \delta$, with $\delta \sim \mathcal{N}(0,\Sigma)$ or $\|\delta\|\le \epsilon$. The following \autoref{thm:lipschitz_noise} shows how this noise affects the Lipschitz constant and the effective sensitivity of the network.

\begin{thm}
\label{thm:lipschitz_noise}
    Under the same assumptions as \autoref{thm:lipschitz_clean}, let $x' = x + \delta$ and assume the Jacobian of $f_\theta$ is locally $L_J$-Lipschitz, i.e.
    \[
        \|J_f(x) - J_f(y)\|_2 \le L_J \|x-y\|_2 \quad \forall x,y.
    \]
    Then, the difference between the noisy and clean network outputs satisfies
    \[
        \|f_\theta(x+\delta) - f_\theta(x)\|_2 \le L_f^{(\mathrm{clean})}\|\delta\|_2 + \tfrac{1}{2} L_J \|\delta\|_2^2.
    \]
    Consequently, the effective Lipschitz constant in the presence of input noise is
    \[
        L_f^{(\mathrm{noisy})} \le L_f^{(\mathrm{clean})} + \tfrac{1}{2} L_J \|\delta\|_2.
    \]
\end{thm}

\begin{proof}
Using the second-order Taylor expansion around $x$:
\[
    f(x+\delta) = f(x) + J_f(x)\delta + R(x,\delta),
\]
where the remainder $R(x,\delta)$ satisfies 
$\|R(x,\delta)\|_2 \le \tfrac{1}{2}L_J\|\delta\|_2^2$ 
by Lipschitz continuity of the Jacobian.
Taking norms and applying the bound for $J_f(x)$ gives
\begin{align*}
    \|f(x+\delta)-f(x)\|_2 
    &\le \|J_f(x)\|\|\delta\|_2 + \tfrac{1}{2}L_J\|\delta\|_2^2\\
    \|f(x+\delta)-f(x)\|_2
    &\le L_f^{(\mathrm{clean})}\|\delta\|_2 + \tfrac{1}{2}L_J\|\delta\|_2^2.
\end{align*}
Dividing by $\|\delta\|_2$ and extending by $x$ yields:
\begin{align*}
    \frac{\|f(x+\delta)-f(x)\|_2}{\|\delta\|_2} \le L_f^{(\mathrm{clean})} + \tfrac{1}{2}L_J\|\delta\|_2,
\end{align*}
and therefore the effective Lipschitz constant under noisy input.
\end{proof}

This result quantifies how input noise modifies the effective sensitivity of the network. The clean Lipschitz constant $L_f^{(\mathrm{clean})}$ acts as the baseline, while the additional term $\tfrac{1}{2}L_J\|\delta\|_2$ represents the curvature-induced sensitivity due to local nonlinearity. Hence, as the input noise amplitude $\|\delta\|_2$ increases, the local Lipschitz constant grows linearly, indicating that noise amplifies sensitivity in regions where the Jacobian varies rapidly.

Practically, this has two crucial implications:
\begin{enumerate}
    \item For correction or denoising tasks: one should minimize both $L_f^{(\mathrm{clean})}$ and $L_J$. This ensures that noise does not destabilize training and that reconstruction remains smooth.
    \item For anomaly detection: A controlled increase of $L_J$ in critical layers can enhance sensitivity to small perturbations, making anomalous deviations more detectable. However, excessive curvature can make training unstable and lead to gradient explosion.
\end{enumerate}

Thus, by controlling $L_f^{(\mathrm{clean})}$ and $L_J$ through appropriate training strategies, one can design networks that either suppress or amplify the propagation of noise, a fundamental mechanism behind robust correction versus sensitive detection architectures.


\section{Multimodal alignment-based failure correction}
\label{sec:method}

The mathematical results from \autoref{thm:comparison} directly translate into architectural design. For effective anomaly detection, the perturbation must remain localized. Convolutional, graph-based, or local-attention layers are ideal for this purpose, as they maintain the geometric or topological coherence of anomalies. For correction or robust averaging, the opposite property is beneficial: spreading the perturbation reduces local intensity and allows recovery by averaging. Dense or strongly mixing architectures act as perturbation diffusers, approximating an isotropic error field where localized anomalies are removed. In the following sections, we introduce our proposed fault-tolerant two-staged self-supervised approach presented in \autoref{fig:overview}. 

\begin{figure*}
    \centering
    \includegraphics[width=\linewidth]{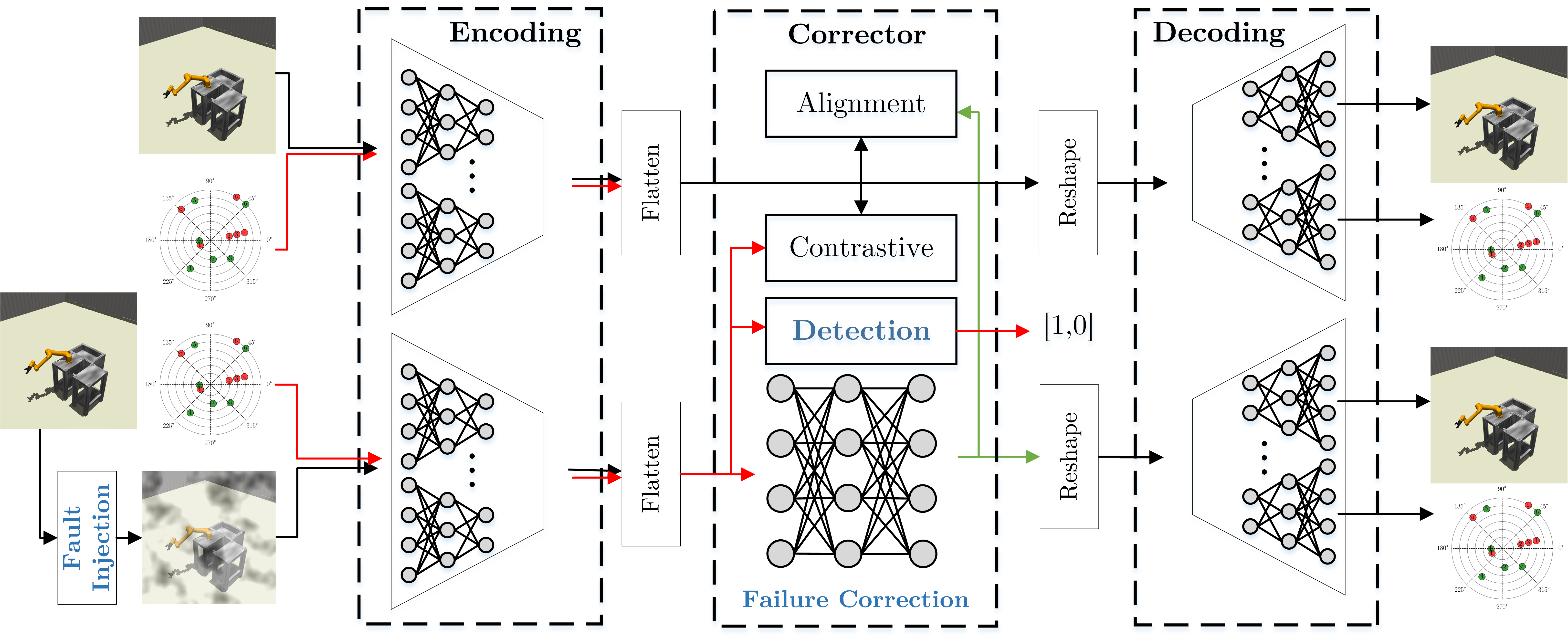}
    \caption{Overview of the proposed fault-tolerant approach for robotic systems. The architecture includes encoding, correction with contrastive learning and detection modules, and decoding stages to reconstruct the original signal.}
    \label{fig:overview}
\end{figure*}

\subsection{Problem Definition}

Let $\mathcal{X} = [\mathcal{X}_i]_{i=1}^M$ denote the multimodal input space, where each modality $\mathcal{X}_i$ corresponds to a distinct sensory or data domain. The model consists of modality-specific encoders $E_i: \mathcal{X}_i \to \mathbb{R}^{z_i}$ and decoders $D_i: \mathbb{R}^{z_i} \to \mathcal{X}_i$, which map the latent spaces back to the input. The latent codes are aggregated into a joint representation via a learnable fusion operator 
\begin{align}
    C: \biguplus_{i=1}^M \mathbb{R}^{z_i} \to \mathbb{R}^z, \quad \tilde{z} = C([E_i(x_i)]_{i=1}^M),
\end{align}
which is designed to detect latent inconsistencies and correct modality-specific faults.

To model corruptions, each modality may undergo a stochastic perturbation $f^{(i)}:\mathcal{X}_i \to \mathcal{X}_i$ such that $x_i^{\mathrm{fail}} = f^{(i)}(x_i, \delta_i)$, where $\delta_i$ represents the disturbance. The objective is to learn parameters $(\theta_i,\phi_i,\psi)$ of $(E_i, D_i, C)$ minimizing the expected multimodal reconstruction loss under corrupted observations, while enforcing contrastive consistency in the latent space:
\begin{align}
    \min_{\theta,\phi,\psi} \mathbb{E}_{x,\delta}\Big[\sum_{i=1}^M &\|x_i - D_i(\tilde{z})\|_2^2 \nonumber \\
    &+ \lambda_1 \mathcal{L}_{\mathrm{con}}(E(x),E(f(x))) \\
    &+ \lambda_2 \mathcal{L}_{\mathrm{sim}}(C(E(f(x))),E(x))
    \Big]. \nonumber
\end{align}
Here, $\mathcal{L}_{\mathrm{con}}$ enforces separation between clean and corrupted representations for anomaly detection, whereas $\mathcal{L}_{\mathrm{sim}}$ promotes post-correction alignment for reconstruction consistency.

\subsection{Pre-trained Multimodal Autoencoder}

Our proposed method in \autoref{fig:overview} is a two-stage framework that utilizes a pre-trained multimodal autoencoder consisting of encoder $E$ and decoder $D$, a corrector ($f(z) \approx f(z + \epsilon)$) for robust data correction, and an anomaly detection. The core components are defined as follows:
\begin{enumerate}
    \item Encoder $E: \mathbb{R}^N \to \mathbb{R}^d$: Maps the high-dimensional input space $\mathbb{R}^N$ to a low-dimensional latent space $\mathbb{R}^d$.
    \item Decoder ($D:\mathbb{R}^d\to\mathbb{R}^N$): Reconstructs the input from the latent code.
    \item Corrector ($C: \mathbb{R}^d \to \mathbb{R}^d$): Corrects corrupted inputs in the latent space.
    \item Detector ($M: \mathbb{R}^d\to[0,1]^d$): Measures the level of anomaly in the latent representation and outputs a score.
\end{enumerate}

Pre-training the autoencoder as a denoising one exposes it to perturbations. While this improves reconstruction robustness, it fundamentally contradicts the requirements for anomaly detection derived in \autoref{sec:gradclip}. As shown in \autoref{thm:lipschitz_noise}, denoising forces the encoder to minimize its sensitivity to input perturbations. However, explicitly suppressing perturbations in $E$ would diminish the anomaly signal in the latent representation $z$, rendering the subsequent detector $M$ ineffective. Therefore, we pretrain the multimodal autoencoder exclusively on clean data. This preserves the encoder's sensitivity to out-of-distribution shifts, ensuring that localized perturbations propagate into the latent space, consistent with the structural analysis in \autoref{thm:comparison}. After pre-training, the weights of encoder $E$ and decoder $D$ are frozen, establishing a fixed, deterministic mapping. With $(E, D)$ fixed, the objective shifts to training the corrector $C$ to invert the effect of a perturbation $\delta$. This decouples the conflicting goals of sensitivity assigned to $E$ and robustness assigned to $C$.

\subsection{Self-supervised Compute for Failure Detection and Correction}
\label{ssec:ssl_compute}

At the core of the proposed architecture lies the compute block \( C: \biguplus_{i=1}^M \mathbb{R}^{z_i} \to \mathbb{R}^z \), which aggregates modality-specific latent representations into a unified, fault-corrected embedding. For each modality \( i \), let
\begin{align}
    z_i = E_i(x_i), \quad z_i^{\mathrm{fail}} = E_i(f^{(i)}(x_i)),
\end{align}
denote the clean and corrupted latent codes, respectively. The clean and corrupted aggregated embeddings are thus
\begin{align}
    z_c = C([z_i]_{i=1}^M), \qquad z_f = C([z_i^{\mathrm{fail}}]_{i=1}^M).
\end{align}

The function \(C\) serves a dual objective: (i) before correction, the representations \( \{z_i^{\mathrm{fail}}\} \) should remain separable from their clean counterparts to enable fault detection, and (ii) after correction, the fused representation \( z_f \) should align closely with \( z_c \) to ensure consistency and recovery. These objectives are enforced via two competing regularizers.

To encourage distinguishability of clean and corrupted states, a contrastive margin loss is applied in the pre-compute space:
\begin{align}
    \mathcal{L}_{\mathrm{con}} = \mathbb{E}\big[\max(0,\, m - \|z_c - z_f\|_2)\big],
\end{align}
where \(m>0\) denotes a minimum separation margin. The loss function, $\mathcal{L}_{\mathrm{con}}$, actively maximizes the $\ell_2$-distance between clean ($z_c$) and corrupted ($z_f$) embeddings, driving this separation to be greater than the margin ($m$). Therefore, minimizing $\mathcal{L}_{\mathrm{con}}$ forces the system to satisfy the constraint $\|z_c - z_f\|_2 \geq m$. 

Under a global Lipschitz constraint, the encoder satisfies
\begin{align}
    \|z_c - z_f\|_2 \le L\|x - (x+\mathbf{n})\|_2 = L\|\mathbf{n}\|_2.
\end{align}
This means that when $L$ is small, the embeddings of $x$ and $x+\mathbf{n}$ remain close, indicating robustness to the perturbation $\mathbf{n}$. In contrast, a margin-based contrastive loss encourages distinguishability by enforcing a minimal separation $m$:
\begin{align}
    \|z_c - z_f\|_2 \ge m \quad (\mathcal{L}_{\mathrm{con}} \to 0).
\end{align}

If the encoder is indeed constrained to be $L$-Lipschitz, these two conditions combine to give
\begin{align}
    m \le \|z_c - z_f\|_2 \le L,\|\mathbf{n}\|_2
\quad\Rightarrow\quad L \ge \frac{m}{\|\mathbf{n}\|_2}.
\end{align}
Thus, under a Lipschitz constraint, the contrastive loss effectively forces the model to allocate a sufficiently large local Lipschitz constant for the specific perturbation pair $(x, x+\mathbf{n})$. In other words, even if $\|\mathbf{n}\|_2$ is small, the representation must change by at least $m$, causing the encoder to become locally sensitive along the direction defined by $\mathbf{n}$. Importantly, this effect applies only when a global Lipschitz bound is enforced. Without such a constraint, the network can achieve the separation $m$ without implying any particular Lipschitz behavior.

After aggregation, the compute block should contract corrupted representations toward the clean manifold:
\begin{align}
    \mathcal{L}_{\mathrm{sim}} = \mathbb{E}\big[\|C([z_i^{\mathrm{fail}}]) - C([z_i])\|_2^2\big],
\end{align}
which minimizes the distortion introduced by perturbations. This term enforces stability and correction capability. The overall loss governing the computational dynamics is
\begin{align}
    \mathcal{L}_C = \lambda_{\mathrm{con}} \mathcal{L}_{\mathrm{con}} + \lambda_{\mathrm{sim}} \mathcal{L}_{\mathrm{sim}},
\end{align}
with weights \(\lambda_{\mathrm{con}}, \lambda_{\mathrm{sim}}>0\) balancing detection sensitivity and correction fidelity.  

Let \(L_C\) denote the Lipschitz constant of \(C\). If \( \lambda_{\mathrm{con}} \gg \lambda_{\mathrm{sim}} \), optimization emphasizes representation separation, increasing \(L_C\) and thus sensitivity to perturbations, which is beneficial for anomaly detection but prone to overreacting to minor noise. Conversely, if \( \lambda_{\mathrm{sim}} \gg \lambda_{\mathrm{con}} \), \(C\) is trained as a contraction mapping (\(L_C < 1\)), ensuring robustness and fault correction but reducing detectability of small anomalies. The optimal trade-off occurs when \(C\) maintains a piecewise Lipschitz behavior: expansive in directions aligned with anomaly subspaces, and contractive along nominal data manifolds. This duality formalizes the balance between contrastive and similarity alignment, allowing the compute block to serve simultaneously as a fault detector and corrector within the multimodal latent space.

\subsection{Layer-Specific Lipschitz Regularization Strategy}

From \autoref{thm:lipschitz_noise}, the local Taylor expansion under input perturbations reveals that the deviation  $|f(x+\delta)-f(x)|$  decomposes into two fundamentally different contributions: a first-order term governed by the Jacobian norm $L_f^{(\mathrm{clean})}=\|J_f(x)\|$, and a second-order term controlled by the Jacobian variation $L_J=\|J_f(x+\delta)-J_f(x)\|$. The former determines the layer’s intrinsic sensitivity to small input changes, while the latter quantifies curvature and noise amplification. These effects play opposing functional roles in anomaly detection, where high sensitivity is beneficial, and correction or reconstruction, where strong contraction and low curvature are required. Therefore, a uniform Lipschitz constraint across the entire network is suboptimal. Instead, the \autoref{thm:lipschitz_noise} motivates a layer-specific regularization strategy in which encoder layers are allowed, and even encouraged, to maintain higher Jacobian norms for enhanced perturbation sensitivity, while compute and reconstruction layers are regularized to enforce low $L_f$ and $L_J$, promoting stability and contraction.

Formally, let the network be partitioned into $N_E$ encoder layers $\{E^{(\ell)}\}_{\ell=1}^{N_E}$ and $N_C$ compute block layers $\{C^{(m)}\}_{m=1}^{N_C}$. Each layer $f^{(\ell)}$ is characterized by its local Jacobian $J^{(\ell)} = \nabla_x f^{(\ell)}(x)$ and corresponding local Lipschitz quantities
\begin{align}
    L_f^{(\ell)} &= \|J^{(\ell)}\|_2, \\ L_J^{(\ell)} &= \mathbb{E}_{x,\delta}\big[\|J^{(\ell)}(x+\delta)-J^{(\ell)}(x)\|_F\big].
\end{align}
The first term measures the local sensitivity of clean inputs, while the second quantifies the curvature or Jacobian variation under noise. Controlling these constants enables architectural tuning between stability and sensitivity.

We assign different regularization objectives depending on the layer role:
\begin{align}
    \mathcal{L}_{\mathrm{reg}} & = \sum_{\ell \in \mathcal{E}} \left( \lambda_{\mathrm{con}}^{(\ell)} \|J^{(\ell)}(x+\delta)-J^{(\ell)}(x)\|_F^2 \right)\\ &+ \sum_{m \in \mathcal{C}} \left( \lambda_{\mathrm{corr}}^{(m)} \|J^{(m)}(x)\|_F^2 \right),\nonumber
\end{align}
where $\mathcal{E}$ and $\mathcal{C}$ denote the encoder and compute-layer index sets, respectively. The weighting coefficients $\lambda_{\mathrm{con}}^{(\ell)}$ and $\lambda_{\mathrm{corr}}^{(m)}$ are selected such that $\lambda_{\mathrm{con}}^{(\ell)} \approx \lambda_{\mathrm{corr}}^{(m)}$, encouraging similar curvature in the encoder and the compute block. Therefore, the overall loss for optimizing the compute block results in:
\begin{align}
    \mathcal{L} = \mathcal{L}_C + \mathcal{L}_{\mathrm{reg}}.
\end{align}

Practically, this analysis leads to two complementary but spatially separated regimes. For the correction and denoising regime, the objective is to minimize both $L_f^{(\mathrm{clean})}$ and $L_J$ to enforce contraction mappings. This is achieved through spectral norm regularization, Jacobian penalties, or gradient clipping in layers beyond the anomaly detection. Formally, gradient clipping constrains the local Jacobian norm via
\begin{align}
    g \leftarrow g \cdot \min\!\left(1, \frac{\tau}{\|g\|_2}\right),
\end{align}
where $\tau>0$ is the clipping threshold, effectively bounding $L_J \le \tau$. This stabilizes training and ensures smooth reconstruction in the latent correction phase.

To enhance the sensitivity of specific layers to localized perturbations for the anomaly detection regime without inducing gradient explosion, we adopt a selective gradient amplification strategy combined with global renormalization. Let $\theta$ denote all network parameters and $\theta_a \subset \theta$ the subset associated with the amplification region. During backpropagation, the gradients corresponding to $\theta_a$ are scaled by an amplification factor $\alpha > 1$, while the overall gradient norm is constrained by a clipping threshold $\tau$. Formally, for a given loss function $\mathcal{L}$, we compute
\begin{align}
    \tilde{g} = \text{clip}\left( g[\theta \setminus \theta_a] \cup g_a, \tau \right),\\
    \text{with } g = \nabla_\theta \mathcal{L}, \quad g_a = \alpha g[\theta_a], \nonumber
\end{align}
where $\text{clip}(g, \tau) = g \cdot \min\!\left(1, \frac{\tau}{\|g\|_2}\right)$. 

This operation amplifies directional gradients in targeted layers to increase the local Lipschitz constant, thereby enhancing the model's sensitivity to small perturbations. Meanwhile, renormalization ensures that the overall gradient energy remains bounded, ensuring training stability. The corresponding pseudocode is presented in \autoref{alg:safe_amplify}. Practically, this mechanism allows the model to emphasize anomaly-relevant features without compromising convergence or introducing instability.

\begin{algorithm}[ht]
\caption{Selective Gradient Scaling and Renormalization.}
\label{alg:safe_amplify}
\begin{algorithmic}[1]
\Require amplification factor $\alpha>1$, clipping threshold $\tau>0$, learning rate $\eta$
\For{each minibatch $(x, y)$}
  \State Compute loss $\mathcal{L} = \mathcal{L}_{\mathrm{task}}(x, y)$
  \State Compute full gradient $g = \nabla_\theta \mathcal{L}$
  \State Identify target parameter subset $\theta_a \subset \theta$
  \State Amplify gradients on $\theta_a$: $g[\theta_a] \leftarrow \alpha \cdot g[\theta_a]$
  \State Compute total norm: $n = \|g\|_2$
  \If{$n > \tau$}
      \State Normalize gradients: $g \leftarrow g \cdot \frac{\tau}{n}$
  \EndIf
  \State Update parameters: $\theta \leftarrow \theta - \eta \, g$
\EndFor
\end{algorithmic}
\end{algorithm}

This layer-wise Lipschitz modulation establishes a directional flow of sensitivity across the architecture: early and middle convolutional encoder layers amplify and preserve perturbations for detection, while the subsequent compute block acts as a Lipschitz contraction, globally smoothing and correcting these perturbations. The resulting gradient field $\nabla_x f(x)$ transitions from expansive to contractive along the network depth, realizing an effective detect–correct framework that is mathematically grounded in controlled Lipschitz geometry.

\section{Industrial Deployment and Practical Application}
\label{sec:application}

In this section, we outline the integration of the proposed MMSSL framework within a real-world industrial production environment. The deployment architecture, illustrated in \autoref{fig:application_workflow}, streamlines industrial integration through a three-stage pipeline comprising data ingestion, intelligent routing, and downstream execution. The system continuously ingests $N$ heterogeneous modalities via dedicated encoders ($E_1, \dots, E_n$) that map high-dimensional signals into a unified latent feature space while preserving perturbation sensitivity. At the core of the framework lies a conditional logic gate driven by a lightweight anomaly detector. Unlike traditional autoencoders that indiscriminately reconstruct all inputs, our approach dynamically routes the data based on its health. If the detector classifies the embedding as clean, the representation bypasses the correction block entirely, minimizing inference latency and preserving original signal fidelity. Conversely, upon detection of faults such as sensor corruption or occlusion, the embedding is routed through the Lipschitz-regularized correction module, which projects the distorted features back onto the learned clean manifold. This mechanism ensures that downstream applications receive a guaranteed valid representation, effectively decoupling high-level process control from low-level sensor reliability.

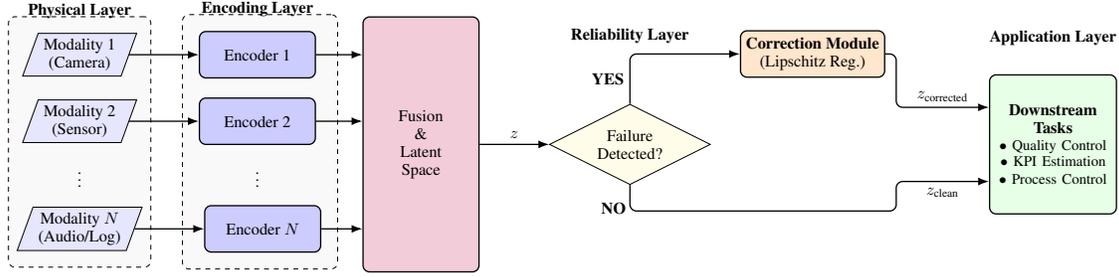
\begin{figure*}[htb]
\centering
\resizebox{0.9\linewidth}{!}{%
\begin{tikzpicture}[
    node distance=1.5cm and 1.5cm,
    box/.style={draw, rounded corners, minimum width=2.5cm, minimum height=1cm, align=center, fill=white, thick},
    sensor/.style={draw, trapezium, trapezium left angle=70, trapezium right angle=110, fill=blue!10, minimum width=2cm, align=center},
    decision/.style={draw, diamond, aspect=2, fill=yellow!10, text width=2cm, align=center, inner sep=1pt},
    arrow/.style={-Latex, thick, rounded corners},
    group/.style={draw, dashed, inner sep=10pt, fill=gray!5, rounded corners}
]

\node[sensor] (mod1) {Modality 1\\(Camera)};
\node[sensor, below=0.5cm of mod1] (mod2) {Modality 2\\(Sensor)};
\node[below=0.3cm of mod2] (dots) {$\vdots$};
\node[sensor, below=0.3cm of dots] (modn) {Modality $N$\\(Audio/Log)};

\node[box, right=1.5cm of mod1, fill=blue!20] (enc1) {Encoder 1};
\node[box, right=1.5cm of mod2, fill=blue!20] (enc2) {Encoder 2};
\node[right= 3.5cm of dots] (dots1) {$\vdots$};
\node[box, right=1.5cm of modn, fill=blue!20] (encn) {Encoder $N$};

\node[box, fill=purple!20, minimum height=5.5cm, right=1cm of enc2, yshift=-0.5cm, text width=1.5cm] (fusion) {Fusion\\\&\\Latent Space};

\node[decision, right=1.5cm of fusion] (dec) {Failure\\Detected?};

\node[box, fill=orange!20, above right=1cm and 1.5cm of dec] (correct) {\textbf{Correction Module}\\(Lipschitz Reg.)};

\coordinate[below right=1cm and 1.5cm of dec] (bypass);

\node[box, fill=green!10, right=6cm of dec, text width=2.5cm, minimum height=3cm] (downstream) {\textbf{Downstream Tasks}\\ \footnotesize $\bullet$ Quality Control\\ $\bullet$ KPI Estimation\\ $\bullet$ Process Control};

\draw[arrow] (mod1) -- (enc1);
\draw[arrow] (mod2) -- (enc2);
\draw[arrow] (modn) -- (encn);

\draw[arrow] (enc1.east) -- (fusion.west |- enc1.east);
\draw[arrow] (enc2.east) -- (fusion.west |- enc2.east);
\draw[arrow] (encn.east) -- (fusion.west |- encn.east);

\draw[arrow] (fusion) -- node[above] {$z$} (dec);

\draw[arrow] (dec.north) |- node[near start, left] {\textbf{YES}} (correct.west);
\coordinate (turn_x) at ($(downstream.west) + (-2, 0)$);
\draw[arrow] (correct.east) -- (correct.east -| turn_x) |- node[near end, above] {$z_{\text{corrected}}$} ($(downstream.west) + (0, 0.8)$);
\draw[arrow] (dec.south) |- (bypass) -- (bypass -| turn_x) |- node[near end, below] {$z_{\text{clean}}$} ($(downstream.west) + (0, -0.8)$);

\node[above=0.2cm of mod1, font=\bfseries] {Physical Layer};
\node[above=0.2cm of enc1, font=\bfseries] {Encoding Layer};
\node[above=1.2cm of dec, font=\bfseries] {Reliability Layer};
\node[below=.25cm of dec, xshift=-0.35cm] {\textbf{NO}};
\node[above=0.5cm of downstream, font=\bfseries] {Application Layer};

\begin{scope}[on background layer]
    \node[fit=(mod1)(modn), group] {};
    \node[fit=(enc1)(encn), group] {};
\end{scope}

\end{tikzpicture}
}
\caption{Operational pipeline for robust industrial deployment. Heterogeneous inputs are encoded into a joint latent space. The representation is analyzed for faults: clean data bypasses the correction block for efficiency, while faulty data is rectified via the Lipschitz-controlled module before executing downstream tasks.}
\label{fig:application_workflow}
\end{figure*}




\section{Evaluation}
\label{sec:eval}

This section details the evaluation of our proposed fault tolerance framework. We begin by introducing the multimodal datasets leveraged for training and testing, along with a description of the experimental setup, including key hyperparameter choices. We then present and discuss the corresponding quantitative results, which validate the efficacy of our approach in achieving superior detection and correction capabilities.

\subsection{Datasets}

Our experiments utilize three distinct multimodal datasets sourced from the industrial robotics domain and are based on the work \cite{altinses2025benchmarking}. These datasets, which are visually introduced in \autoref{fig:dataset}, are crucial because they act as digital twins of real-world industrial machinery. They are built on authentic operating data and precisely model the physics of actual sensors and systems. Crucially, they include varying levels of task difficulty and structural diversity. Formally, each dataset is composed of $N$ total data samples. A single data sample at index $j$ is a collection of $M$ observations, $(\mathbf{x}_j^{(1)}, \ldots, \mathbf{x}_j^{(M)})$, taken at the same time. The data from each of the $M$ different sensor types has its own space, $\mathcal{M}_i \subseteq \mathbb{R}^{d_i}$, with the specific sensor reading denoted as $\mathbf{x}_j^{(i)} \in \mathcal{M}_i$.

\begin{figure}[htb]
     \centering
     \begin{subfigure}[b]{0.25\linewidth}
         \centering
         \includegraphics[width=\linewidth]{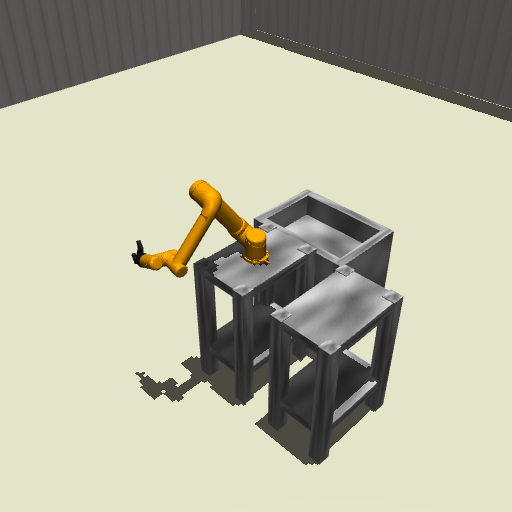}
         \caption{MuJoCo}
         \label{fig:d1}
     \end{subfigure}
     \hfill
     \begin{subfigure}[b]{0.25\linewidth}
         \centering
         \includegraphics[width=\linewidth]{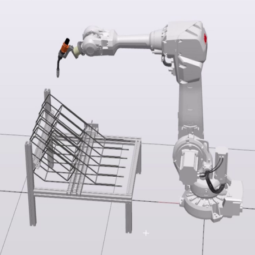}
         \caption{ABB single robot}
         \label{fig:d2}
     \end{subfigure}
     \hfill
     \begin{subfigure}[b]{0.25\linewidth}
         \centering
         \includegraphics[width=\linewidth]{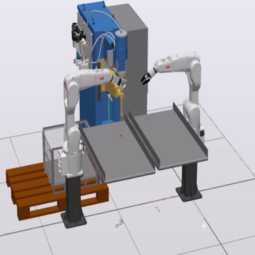}
         \caption{ABB dual robot}
         \label{fig:d3}
     \end{subfigure}
    \caption{Three image modality samples of the three distinct multimodal datasets \cite{altinses2025benchmarking}.}
    \label{fig:dataset}
\end{figure}

In this work, we focus on combining and reconstructing two specific types of paired multimodal data: spatial image data $S$ and kinematic signals $K$. Each of these two data modalities is indexed by $m \in \{s, k\}$. The kinematic sequences $K$, which represent motion or force, are straightforwardly encoded as matrices in $\mathbb{R}^{d_k}$, where $d_k$ is the number of sensor dimensions recorded at each point in time. In contrast, the visual observations $S$ are structured as standard, static RGB images. We use a fixed spatial resolution where the row and column sizes are equal, $s_r = s_c = 256$, meaning the images are represented as tensors in $\mathbb{R}^{3 \times s_r \times s_c}$. 

\subsection{Failure Injection}

The fault injection process is designed based on the methodologies established in the work by \cite{altinses2023multimodal}, which simulate real-world degradation mechanisms. Specifically, the resulting failures for the image modality are visually presented in \autoref{fig:augmented_cams}. In industrial applications, such degradation patterns often correspond to occlusions or mechanical wear, where faults manifest globally and locally with minimum one clean modality. This approach ensures that the network is exposed to representative multimodal failure modes, enabling it to learn fault dependencies during training.

\begin{figure}[htb]
     \centering
     \begin{subfigure}[b]{0.24\linewidth}
         \centering
         \includegraphics[width=\linewidth]{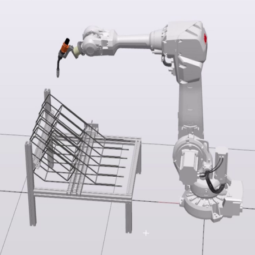}
     \end{subfigure}
     \hfill
     \begin{subfigure}[b]{0.24\linewidth}
         \centering
         \includegraphics[width=\linewidth]{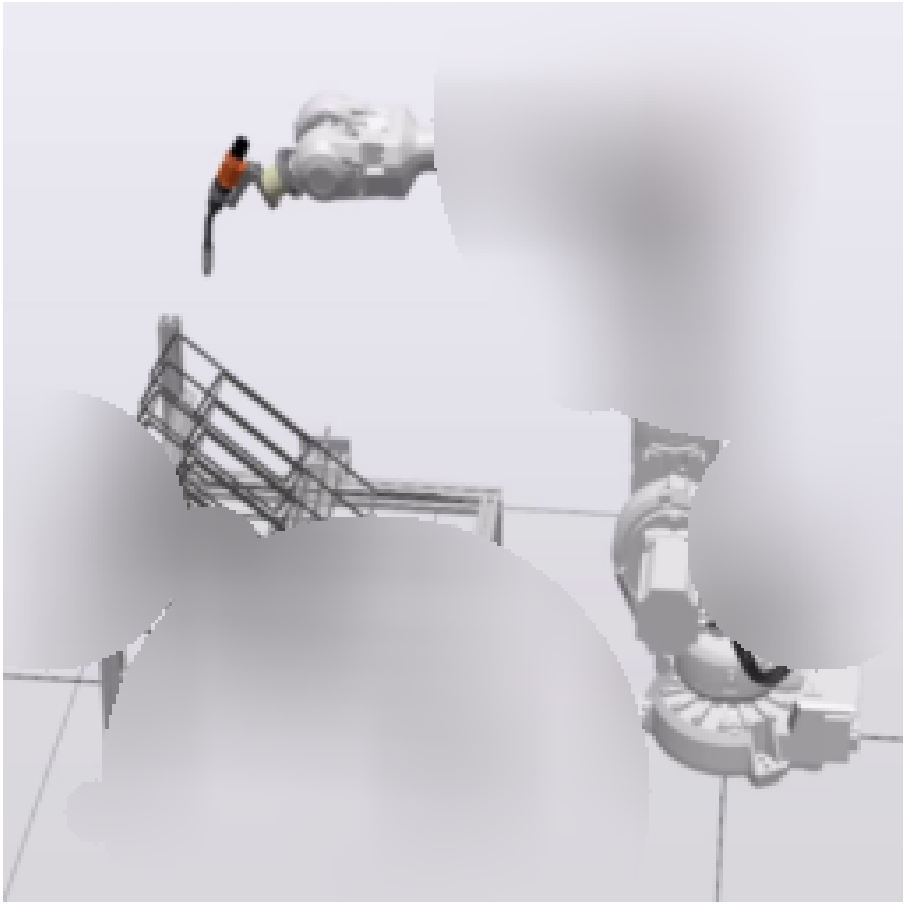}
     \end{subfigure}
     \hfill
     \begin{subfigure}[b]{0.24\linewidth}
         \centering
         \includegraphics[width=\linewidth]{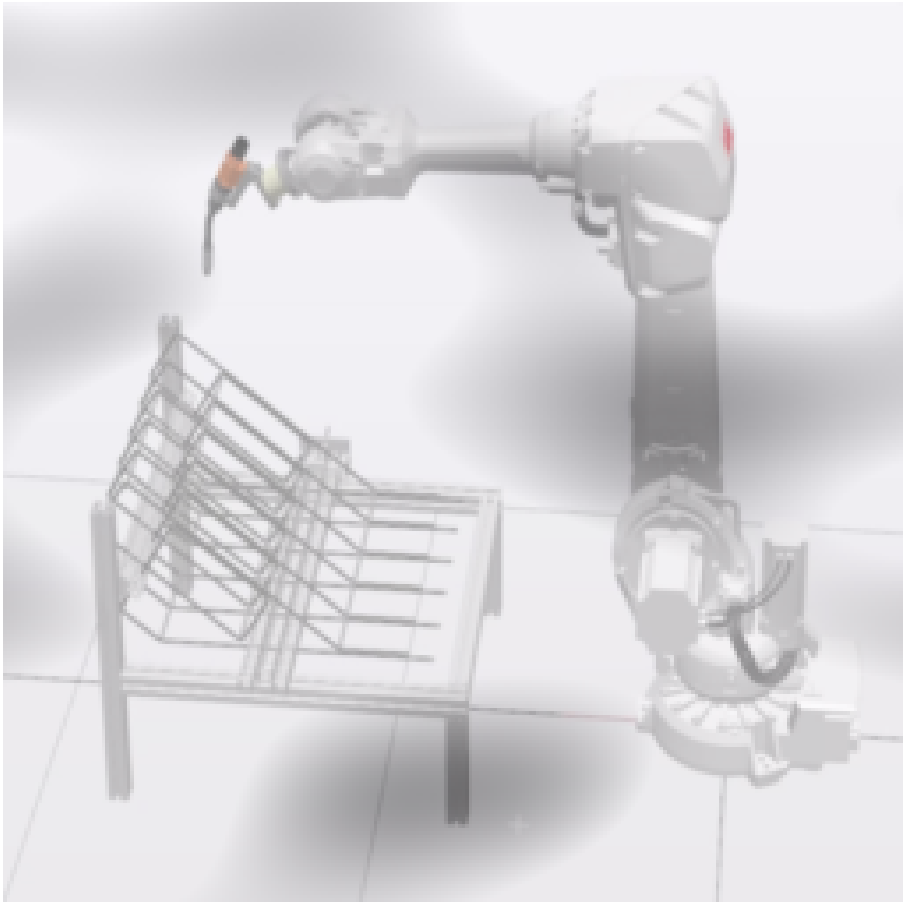}
     \end{subfigure}
     \hfill
     \begin{subfigure}[b]{0.24\linewidth}
         \centering
         \includegraphics[width=\linewidth]{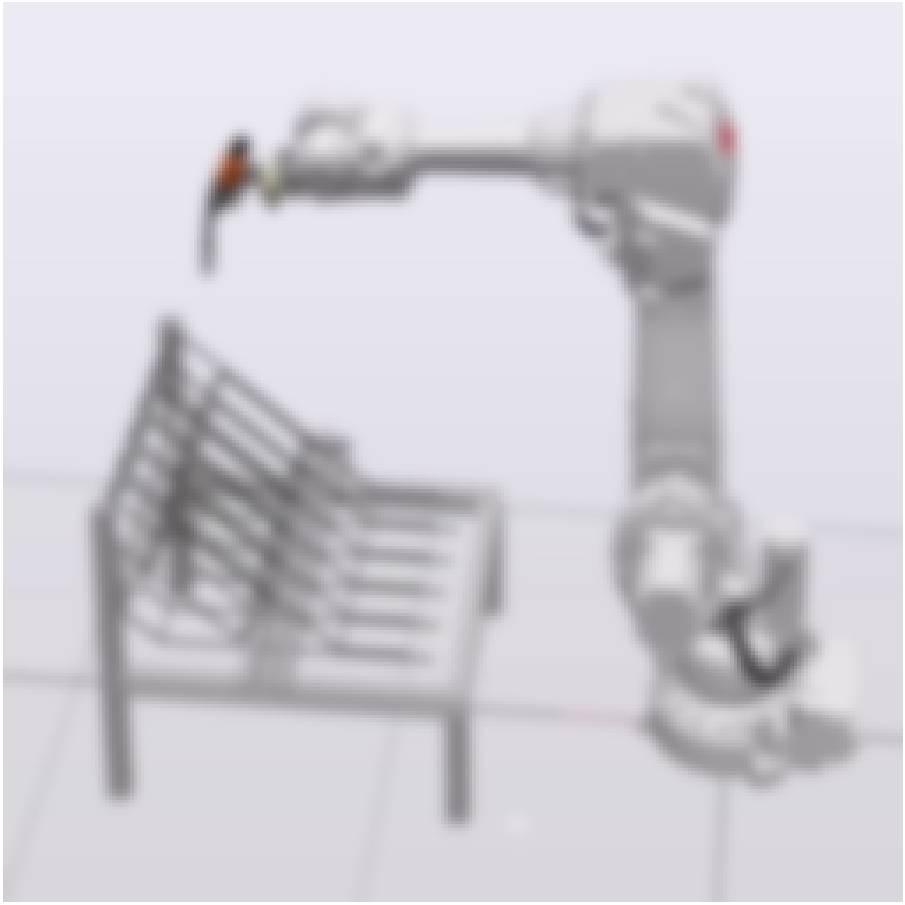}
     \end{subfigure}
     \begin{subfigure}[b]{0.24\linewidth}
         \centering
         \includegraphics[width=\linewidth]{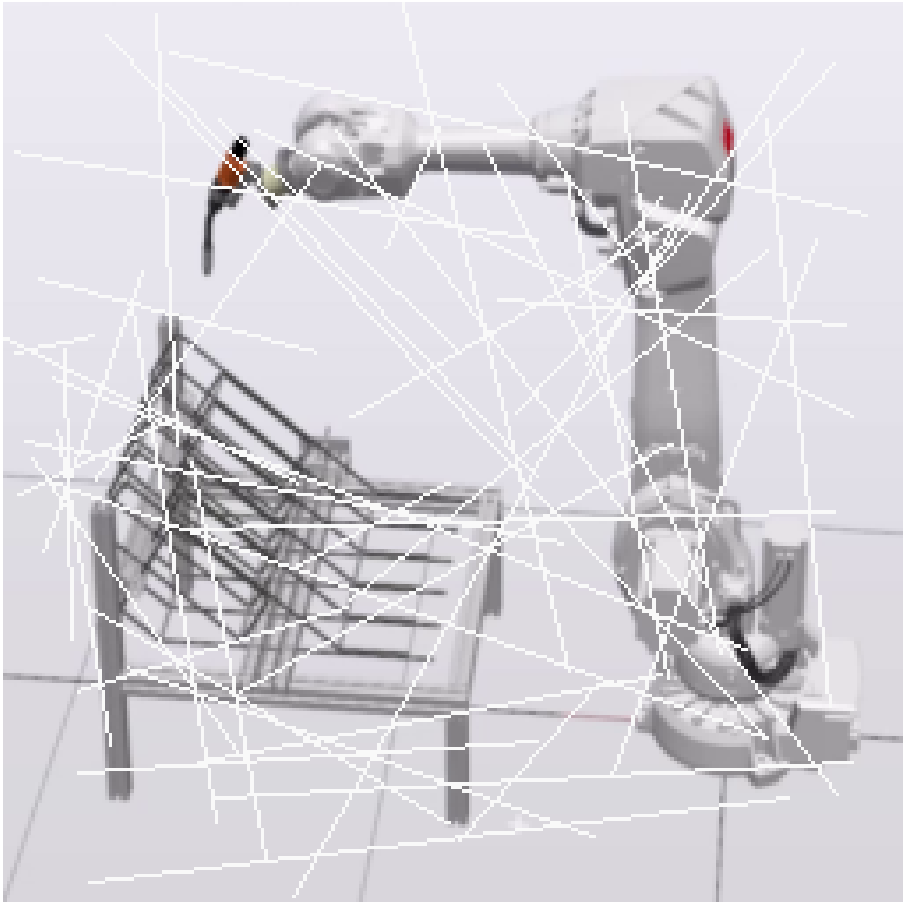}
     \end{subfigure}
     \hfill
     \begin{subfigure}[b]{0.24\linewidth}
         \centering
         \includegraphics[width=\linewidth]{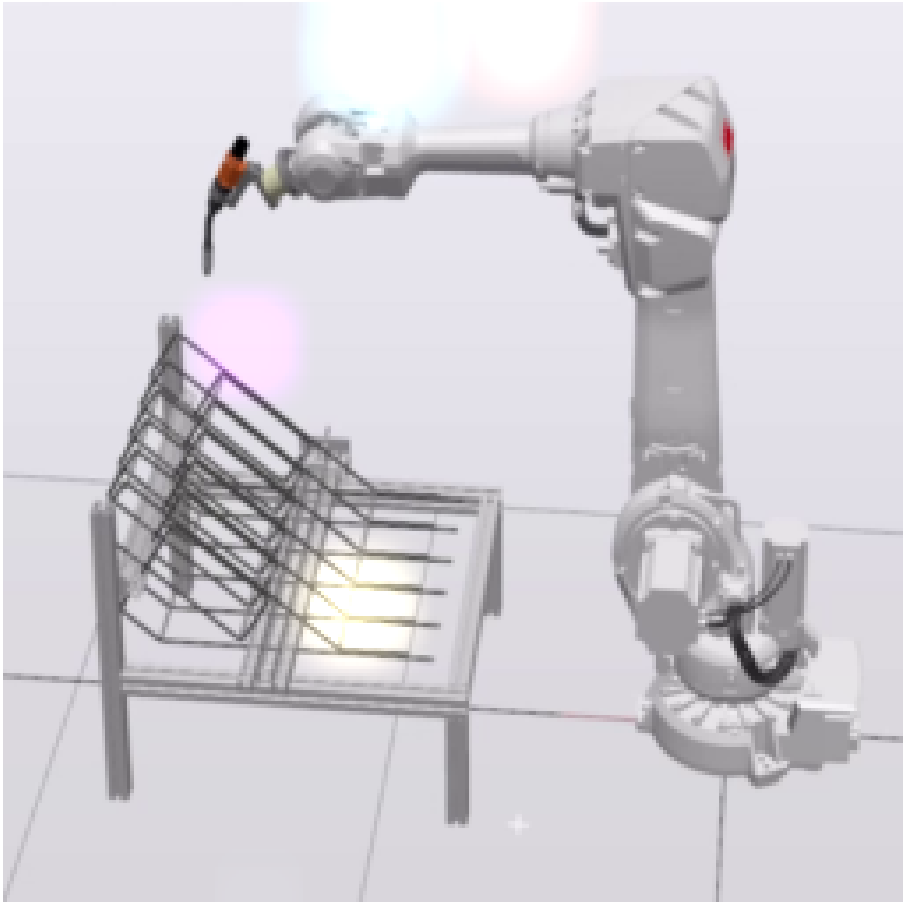}
     \end{subfigure}
     \hfill
     \begin{subfigure}[b]{0.24\linewidth}
         \centering
         \includegraphics[width=\linewidth]{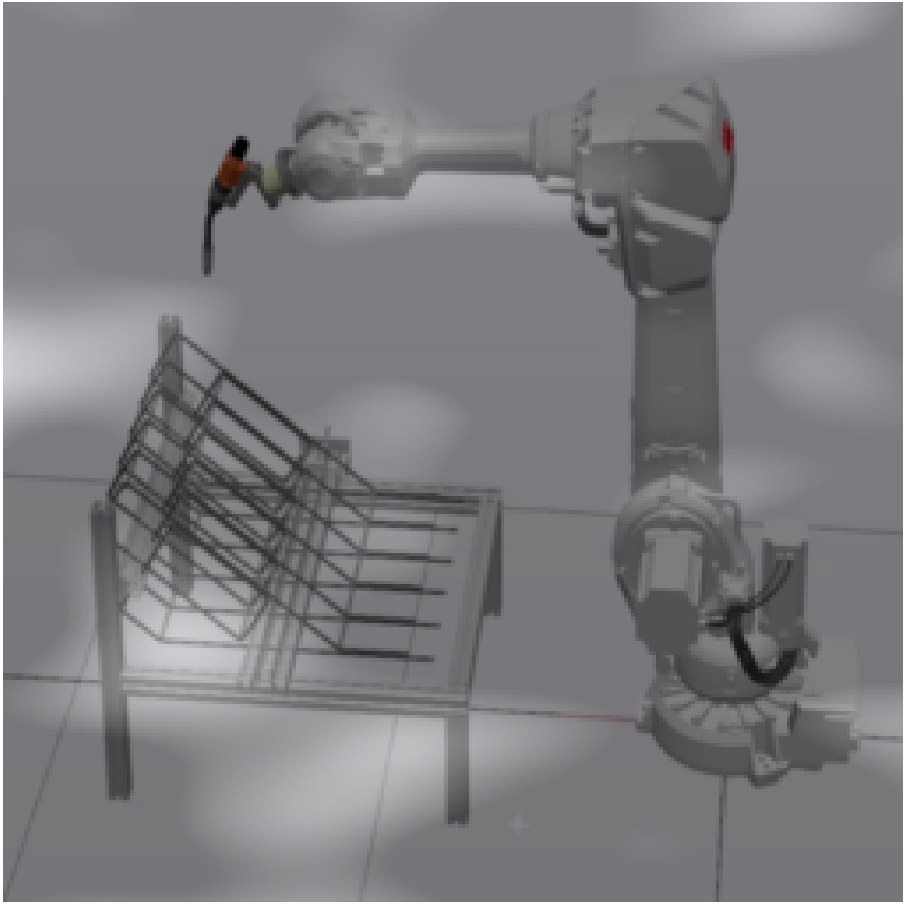}
     \end{subfigure}
     \hfill
     \begin{subfigure}[b]{0.24\linewidth}
         \centering
         \includegraphics[width=\linewidth]{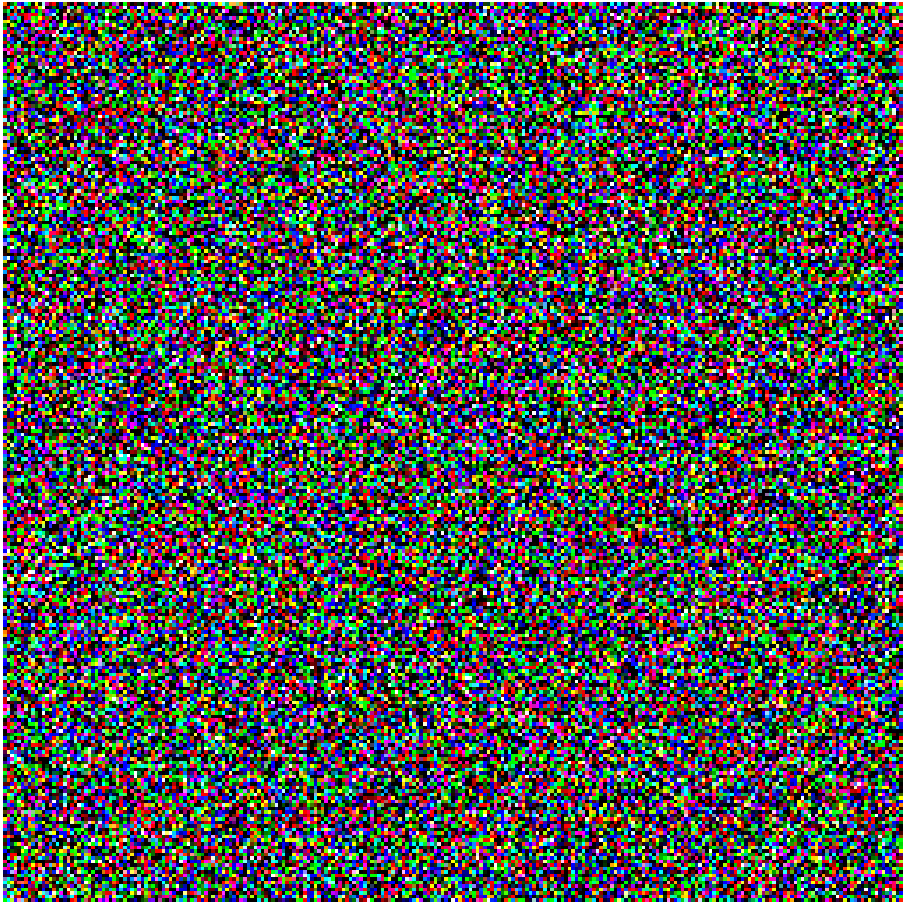}
     \end{subfigure}
     \begin{subfigure}[b]{0.24\linewidth}
         \centering
         \includegraphics[width=\linewidth]{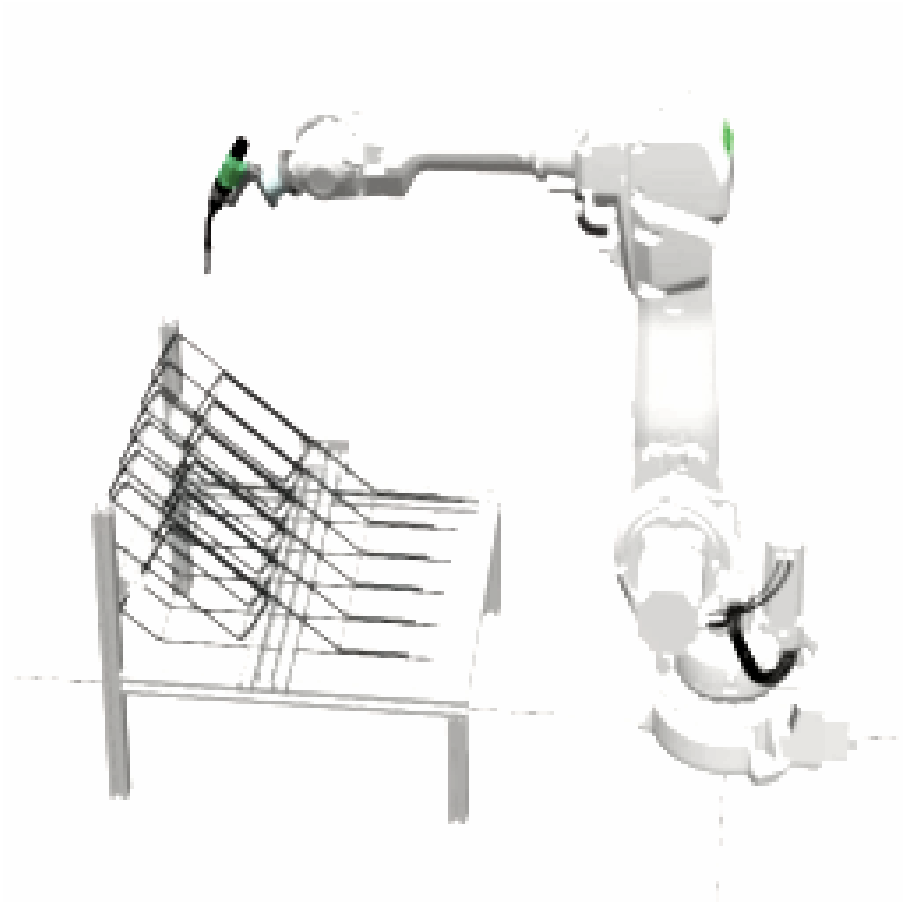}
     \end{subfigure}
     \hfill
     \begin{subfigure}[b]{0.24\linewidth}
         \centering
         \includegraphics[width=\linewidth]{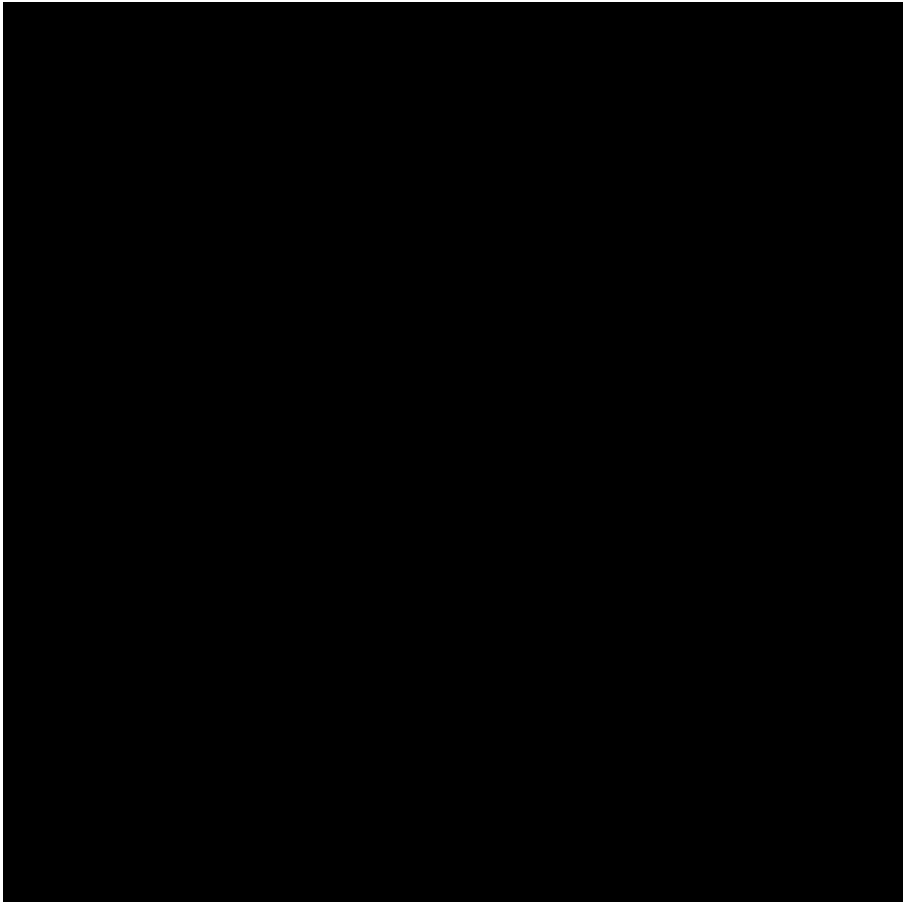}
     \end{subfigure}
     \hfill
     \begin{subfigure}[b]{0.24\linewidth}
         \centering
         \includegraphics[width=\linewidth]{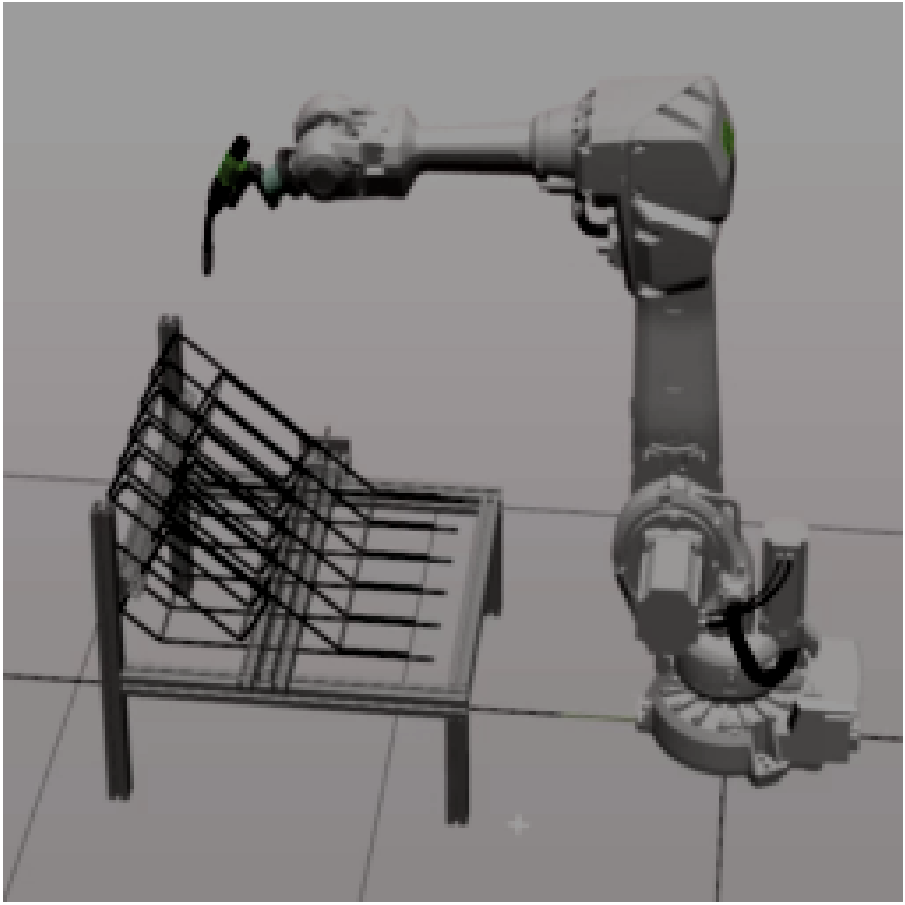}
     \end{subfigure}
     \hfill
     \begin{subfigure}[b]{0.24\linewidth}
         \centering
         \includegraphics[width=\linewidth]{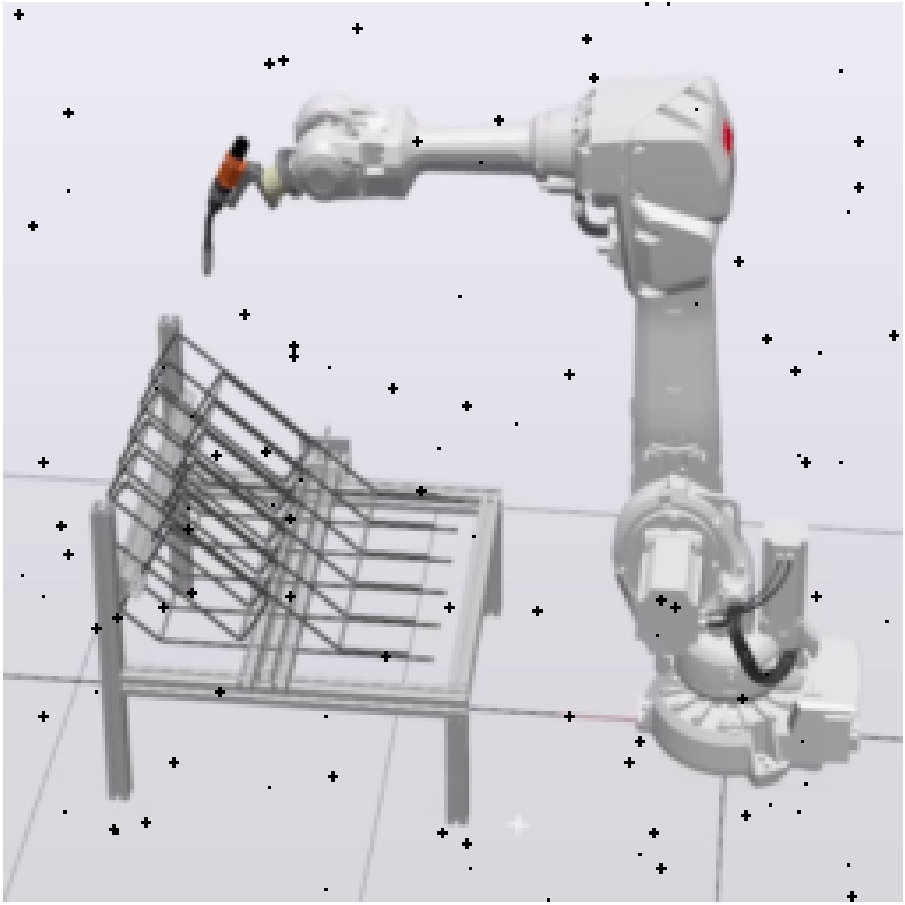}
     \end{subfigure}
    \caption{Several randomly applied augmentation techniques on one random camera sample.}
    \label{fig:augmented_cams}
\end{figure}

\subsection{Experimental Setup}
\label{ssec:exp_setup}

To validate the MMSSL framework for robotic sensory data, we implement the architecture in \autoref{fig:overview} using visual and proprioceptive branches connected via a correction manifold. Preserving anomaly topology as described in \autoref{sec:method}, the spatial encoder $E_1$ employs four 2D convolutional layers (channels 32-256-128-64-32-32-32, kernel $5 \times 5$, stride 2, padding 1). The sensor encoder $E_2$ uses a 1D convolutional network to maintain locality, processing dimension $n_0$ through five layers with kernel size of 3 and stride of 1 to the depth of 288. Latent vectors are concatenated and reconstructed by decoders $D_1, D_2$, mirroring the encoders. The compute block comprises four fully connected layers with size of 576 and ReLU activations, but with the last layer transforming back to the shape of 288. Training proceeds in two isolated stages as described in \autoref{sec:gradclip}. First, the multimodal autoencoder is pre-trained on unperturbed data for 200 epochs using Adam optimizer and the configuration of $\eta_{\mathrm{AE}} = 10^{-4}$, $\beta_1=0.9, \beta_2=0.999$. Second, the compute block $C$ is trained to map corrupted to clean embeddings for 200 epochs with batchsize of 16 to minimize $\mathcal{L}_{\mathrm{total}} = \mathcal{L}_{\mathrm{rec}} + \lambda_1 \mathcal{L}_{\mathrm{con}} + \lambda_2 \mathcal{L}_{\mathrm{sim}}$, with $\lambda_1 = 0.1, \lambda_2 = 1.0$, and margin $m=0.5$. Lipschitz constraints enforce contraction via clipping $\tau = 1.0$ and sensitivity via amplification $\alpha = 3.0$.

\subsection{Results}

This section reports the experimental results of our proposed framework, providing a comparative analysis with state-of-the-art denoising models, including multimodal autoencoder (MMAE) \cite{altinses2025enhancing}, multimodal variational autoencoder (MMVAE) \cite{altinses2025enhancing}, multimodal fuzzy regularization (MMFR) \cite{altinses2023deep}, multimodal concept fusion (MMCF) \cite{altinses2025fault}, multimodal GAN correction (MMGAN) \cite{altinses2026generative}, and multimodal U-GAN correction (MMUGAN) \cite{altinses2026generative}. Our evaluation is structured into four subsections. We begin by assessing the overall performance, followed by an in-depth investigation of the convergence behavior. Then, we discuss the fault detection capabilities, a metric that is intrinsic to our model's architecture, and finally, we prove empirically the theory from \autoref{sec:pertub}.

\subsubsection{Failure Correction Performance}

We begin our analysis by examining the general reconstruction fidelity of the models under failure conditions, with a specific focus on benchmarking our proposed self-supervised learning framework. Unlike generative adversarial approaches that rely on complex discriminator dynamics or variational methods that optimize a probabilistic lower bound, our approach leverages self-supervised geometric constraints to guide the correction process. \autoref{tab:mujoco} details the training and testing performance on the MuJoCo dataset, comparing these distinct paradigms against other state-of-the-art approaches. The table explicitly decomposes the error contributions into high-dimensional camera data and low-dimensional sensor data, providing a detailed view of how effectively the self-supervised signal competes with the state-of-the-art in multimodal failure correction.

\begin{table}[htb]
\centering
\caption{Train and test performances (in $10^{-2}$) of the failure correction using the mean over the final 250 batches of 10 trials. The reconstruction performance of the camera, the sensor, and both combined within the MuJoCo dataset.}
\begin{tabular}{l |c c c }
\toprule
\textbf{Model} & \textbf{Camera} & \textbf{Sensor} & \textbf{Combined} \\
\midrule
\multicolumn{4}{c}{\textbf{Training}}\\
\midrule
MMAE  &  0.114 $\pm$  0.063  &  0.815 $\pm$  0.562  &  0.929 $\pm$  0.584 \\
MMVAE &  0.137 $\pm$  0.026  &  0.694 $\pm$  0.435  &  0.831 $\pm$  0.442 \\
MMFR  &  0.129 $\pm$  0.033  &  0.973 $\pm$  0.500  &  1.102 $\pm$  0.504 \\
MMCF  &  0.113 $\pm$  0.017  &  0.165 $\pm$  0.077  &  0.279 $\pm$  0.079 \\
MMGAN &  0.113 $\pm$  0.025  &  0.158 $\pm$  0.097  &  0.209 $\pm$  0.101 \\
MMUGAN&  0.123 $\pm$  0.027  &  0.257 $\pm$  0.182  &  0.313 $\pm$  0.186 \\
MMSSL &  0.098 $\pm$  0.023  &  0.307 $\pm$  0.302  &  0.405 $\pm$  0.310 \\
\midrule
\multicolumn{4}{c}{\textbf{Testing}}\\
\midrule
MMAE  &  0.145 $\pm$  0.060  &  1.453 $\pm$  0.214  &  1.599 $\pm$  0.183 \\
MMVAE &  0.164 $\pm$  0.011  &  1.395 $\pm$  0.136  &  1.560 $\pm$  0.131 \\
MMFR  &  0.142 $\pm$  0.023  &  1.318 $\pm$  0.121  &  1.461 $\pm$  0.123 \\
MMCF  &  0.154 $\pm$  0.006  &  0.894 $\pm$  0.122  &  1.048 $\pm$  0.121 \\
MMGAN &  0.134 $\pm$  0.005  &  0.894 $\pm$  0.055  &  0.950 $\pm$  0.056 \\
MMUGAN&  0.125 $\pm$  0.002  &  0.779 $\pm$  0.018  &  0.835 $\pm$  0.019 \\
MMSSL &  \bestvalue{0.118 $\pm$  0.005}  &  \bestvalue{0.405 $\pm$  0.040}  &  \bestvalue{0.523 $\pm$  0.042} \\
\bottomrule
\end{tabular}
\label{tab:mujoco}
\end{table}

The quantitative evaluation on the MuJoCo dataset in \autoref{tab:mujoco} demonstrates the superior generalization capability of the proposed MMSSL framework. While generative baselines like MMGAN achieve competitive training fits, they struggle to sustain this performance during testing, likely due to overfitting on the noise distribution. In contrast, MMSSL achieves a state-of-the-art combined test error of 0.523, reducing the error of the closest competitor (MMUGAN, $0.835$) by approximately 37\%. Our method dominates in the sensor modality ($0.405$ compared to $0.779$ for MMUGAN), indicating that the model capture underlying kinematic constraints more effectively. Furthermore, the minimal standard deviation ($\pm$ 0.042) confirms that the Lipschitz-constrained correction acts as a stable contraction mapping, ensuring deterministic recovery and avoiding the training volatility and mode collapse often observed in the adversarial baselines.

To validate this improvement and verify the geometric alignment of the corrected states, we visualize the topology of the learned latent manifold. The following t-SNE embedding in \autoref{fig:tsne} illustrates the distribution of clean, faulty, and corrected representations, demonstrating the model's ability to map perturbed samples back onto the nominal data manifold.

\begin{figure}[htb]
     \centering
     \begin{subfigure}[t]{0.24\linewidth}
         \centering
         \includegraphics[width=\linewidth]{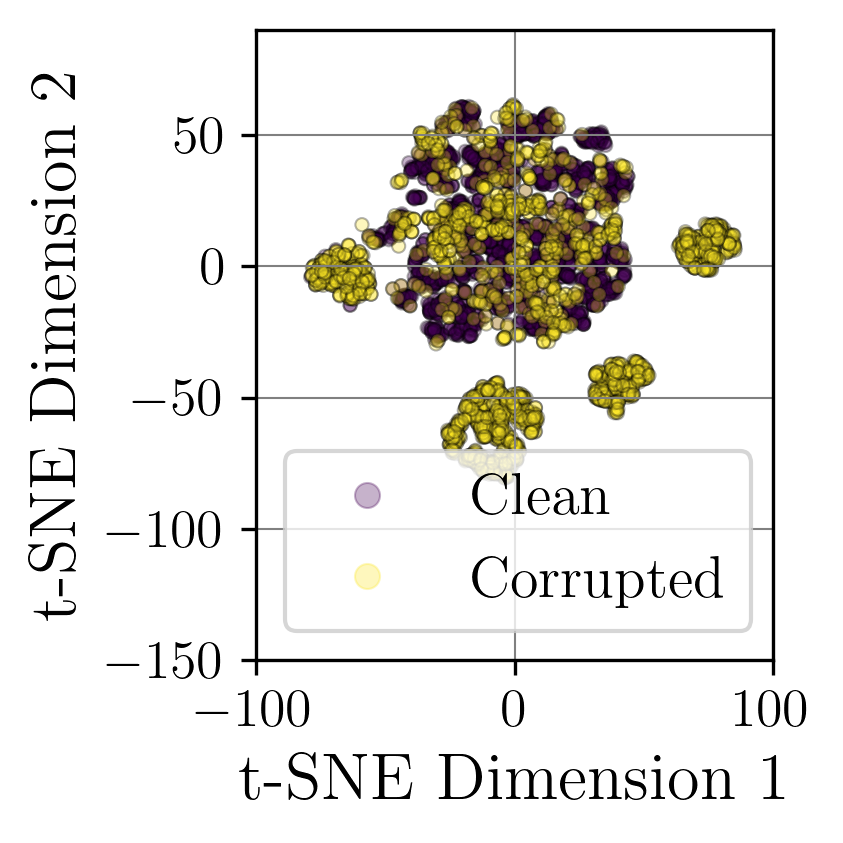}
         \caption{Clean - Corrupted}
         \label{sfig:tsne1}
     \end{subfigure}
     \begin{subfigure}[t]{0.24\linewidth}
         \centering
         \includegraphics[width=\linewidth]{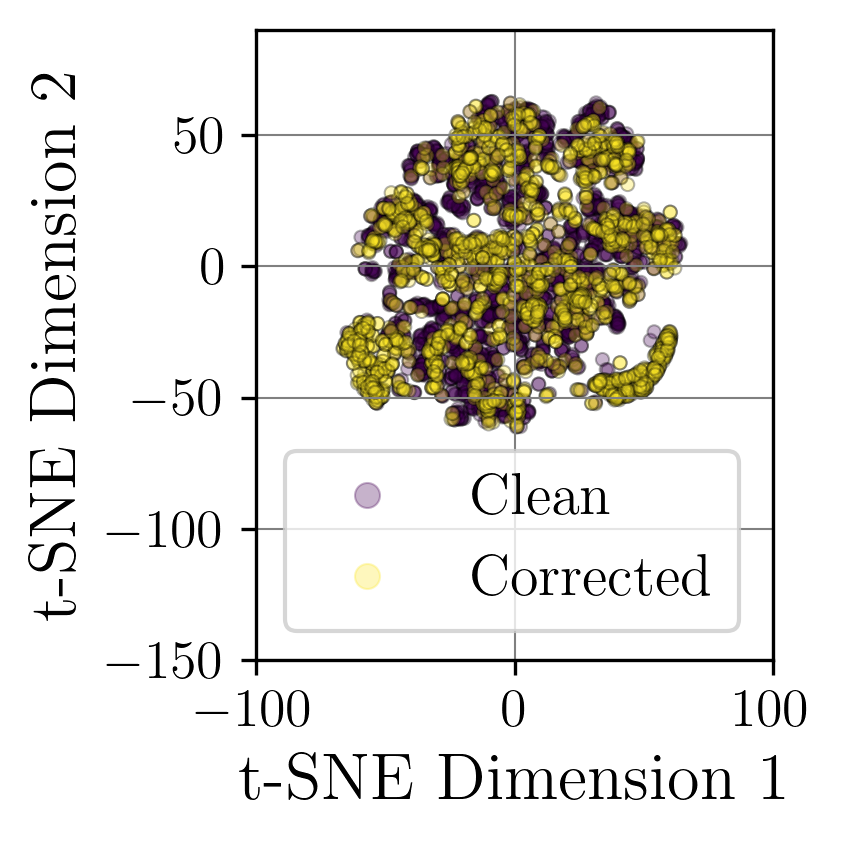}
         \caption{Clean - Corrected}
         \label{sfig:tsne2}
     \end{subfigure}
    \caption{Visualization of the learned latent feature space using t-SNE. (a) Distribution of clean (purple) versus corrupted (yellow) samples before correction. (b) Distribution of clean (purple) versus corrected (yellow) samples.}
    \label{fig:tsne}
\end{figure}

As illustrated, the corrected representations in \autoref{sfig:tsne2} exhibit a near-perfect distributional overlap with the clean baseline, effectively covering the nominal manifold. This confirms that the MMSSL framework successfully projects faulty inputs back onto the valid data distribution. Crucially, the distinct cluster structure of the clean data is fully preserved during this projection. This indicates that the correction process restores physical consistency without collapsing the latent space or discarding semantic information, effectively acting as a precise geometric projection operator.

To validate the scalability and robustness of the proposed framework, we extend our evaluation to a more complex and realistic scenario: the single-robot welding station dataset. This dataset introduces higher variability in sensor noise and more intricate kinematic constraints compared to MuJoCo. \autoref{tab:single} summarizes the reconstruction performance across all baseline models for this setup. 

\begin{table}[htb]
\centering
\caption{Train and test performances (in $10^{-2}$) of the failure correction using the mean over the final 250 batches of 10 trials. The reconstruction performance of the camera, the sensor, and both combined within the single-robot welding station dataset.}
\begin{tabular}{l |c c c }
\toprule
\textbf{Model} & \textbf{Camera} & \textbf{Sensor} & \textbf{Combined} \\
\midrule
\multicolumn{4}{c}{\textbf{Training}}\\
\midrule
MMAE   & 0.131 $\pm$ 0.121 & 1.022 $\pm$ 1.203 & 1.153 $\pm$ 1.282 \\
MMVAE  & 0.158 $\pm$ 0.093 & 0.961 $\pm$ 1.131 & 1.119 $\pm$ 1.180 \\
MMFR   & 0.142 $\pm$ 0.081 & 0.927 $\pm$ 1.092 & 1.069 $\pm$ 1.133 \\
MMCF   & 0.081 $\pm$ 0.022 & 0.602 $\pm$ 0.651 & 0.683 $\pm$ 0.643 \\
MMGAN  & 0.074 $\pm$ 0.035 & 0.563 $\pm$ 0.997 & 0.637 $\pm$ 1.004 \\
MMUGAN & 0.073 $\pm$ 0.016 & 0.540 $\pm$ 0.531 & 0.613 $\pm$ 0.538 \\
MMSSL  & 0.079 $\pm$ 0.032 & 0.478 $\pm$ 1.044 & 0.557 $\pm$ 1.048 \\
\midrule
\multicolumn{4}{c}{\textbf{Testing}}\\
\midrule
MMAE  & 0.141 $\pm$ 0.120 & 1.223 $\pm$ 0.681 & 1.364 $\pm$ 0.802 \\
MMVAE & 0.159 $\pm$ 0.092 & 1.050 $\pm$ 0.561 & 1.209 $\pm$ 0.653 \\
MMFR  & 0.138 $\pm$ 0.083 & 1.071 $\pm$ 0.582 & 1.209 $\pm$ 0.661 \\
MMCF  & 0.091 $\pm$ 0.011 & 0.702 $\pm$ 0.120 & 0.793 $\pm$ 0.122 \\
MMGAN & 0.084 $\pm$ 0.031 & 0.589 $\pm$ 0.476 & 0.673 $\pm$ 0.478 \\
MMUGAN& 0.081 $\pm$ 0.007 & 0.598 $\pm$ 0.072 & 0.679 $\pm$ 0.074 \\
MMSSL & \bestvalue{0.079 $\pm$ 0.003} & \bestvalue{0.496 $\pm$ 0.085} & \bestvalue{0.575 $\pm$ 0.088} \\
\bottomrule
\end{tabular}
\label{tab:single}
\end{table}

The results on the single-robot welding dataset confirm the robustness of the MMSSL framework in complex industrial environments. Our self-supervised approach achieves a state-of-the-art combined reconstruction error of 0.575, significantly outperforming the best generative baseline (MMGAN: 0.673) by a margin of approximately 15\%. MMSSL demonstrates superior reliability in the sensor modality again, reducing the error to 0.496 compared to 0.589 for MMGAN. While adversarial methods achieve competitive pixel-level camera errors, they exhibit significant volatility in the combined metrics (e.g., MMGAN $\pm$ 0.478), indicating a failure to consistently capture multimodal correlations. In contrast, MMSSL enforces semantic consistency through its geometric objective, ensuring that the model prioritizes physical reliability. This results in a stable contraction mapping ($\pm$ 0.088) that effectively filters heavy industrial noise without the stochastic hallucinations typical of generative approaches.

Finally, we scale the evaluation to the most complex scenario: the dual-robot welding station. This dataset introduces intricate inter-agent dependencies and doubles the sensor dimensionality, significantly increasing the difficulty of modeling kinematic correlations. \autoref{tab:dual} summarizes the results, revealing a critical failure mode in baseline architectures. 

\begin{table}[htb]
\centering
\caption{Train and test performances (in $10^{-2}$) of the failure correction using the mean over the final 250 batches of 10 trials. The reconstruction performance of the camera (C), the sensor (S), and both combined (B) within the dual-robot welding station dataset.}
\begin{tabular}{l |c c c }
\toprule
\textbf{Model} & \textbf{Camera} & \textbf{Sensor} & \textbf{Combined} \\
\midrule
\multicolumn{4}{c}{\textbf{Training}}\\
\midrule
MMAE  & 0.031 $\pm$ 0.032 & 1.122 $\pm$ 1.241 & 1.153 $\pm$ 1.273 \\
MMVAE & 0.030 $\pm$ 0.011 & 0.962 $\pm$ 1.092 & 0.992 $\pm$ 1.103 \\
MMFR  & 0.051 $\pm$ 0.041 & 1.121 $\pm$ 1.240 & 1.172 $\pm$ 1.281 \\
MMCF  & 0.021 $\pm$ 0.010 & 0.872 $\pm$ 1.011 & 0.893 $\pm$ 1.021 \\
MMGAN & 0.078 $\pm$ 0.024 & 0.767 $\pm$ 0.562 & 0.844 $\pm$ 0.564 \\
MMUGAN& 0.077 $\pm$ 0.023 & 0.779 $\pm$ 0.729 & 0.856 $\pm$ 0.731 \\
MMSSL & 0.076 $\pm$ 0.024 & 0.683 $\pm$ 2.063 & 0.759 $\pm$ 2.069 \\
\midrule
\multicolumn{4}{c}{\textbf{Testing}}\\
\midrule
MMAE  & 0.032 $\pm$ 0.031 & 1.251 $\pm$ 0.322 & 1.283 $\pm$ 0.353 \\
MMVAE & 0.030 $\pm$ 0.001 & 1.082 $\pm$ 0.162 & 1.112 $\pm$ 0.163 \\
MMFR  & 0.054 $\pm$ 0.031 & 1.274 $\pm$ 0.341 & 1.328 $\pm$ 0.372 \\
MMCF  & \bestvalue{0.020 $\pm$ 0.001} & 0.981 $\pm$ 0.211 & 1.001 $\pm$ 0.212 \\
MMGAN & 0.080 $\pm$ 0.006 & 0.866 $\pm$ 0.142 & 0.945 $\pm$ 0.143 \\
MMUGAN& 0.077 $\pm$ 0.003 & 0.907 $\pm$ 0.167 & 0.984 $\pm$ 0.165 \\
MMSSL & 0.078 $\pm$ 0.002 & \bestvalue{0.777 $\pm$ 0.154} & \bestvalue{0.855 $\pm$ 0.152} \\
\bottomrule
\end{tabular}
\label{tab:dual}
\end{table}

The evaluation on the dual-robot dataset reveals a critical insight regarding multimodal learning dynamics. Baselines such as MMCF and MMVAE achieve superior pixel-level reconstruction (e.g., MMCF: $0.020$), yet this comes at the cost of poor sensor fidelity ($0.981$). This disparity suggests a modality collapse, where the model minimizes the global loss by overfitting to the high-dimensional visual stream while neglecting the lower-dimensional, yet physically critical, sensor data. In contrast, MMSSL prioritizes consistency over superficial visual sharpness. While its camera reconstruction error is higher, it achieves the lowest sensor error ($0.777$), resulting in the best overall combined performance ($0.855$). By outperforming the closest generative baseline (MMGAN) by approximately 10\% and the best autoencoder baseline (MMCF) by 15\%, MMSSL demonstrates that its self-supervised objective successfully fuses the complex dual-robot kinematic streams, preventing visual dominance and ensuring physically grounded fault correction.

\subsubsection{Convergence and Lipschitz behaviour}

Having established the superior failure correction performance of MMSSL across multiple datasets, we now turn our attention to the underlying optimization dynamics that enable this performance. While the previous section quantified the final model performance, this section analyzes the convergence stability and the spectral properties of the learned networks.

\begin{figure}[htb]
    \centering
    \begin{subfigure}[b]{0.24\linewidth}
        \centering
        \includegraphics[width=\linewidth]{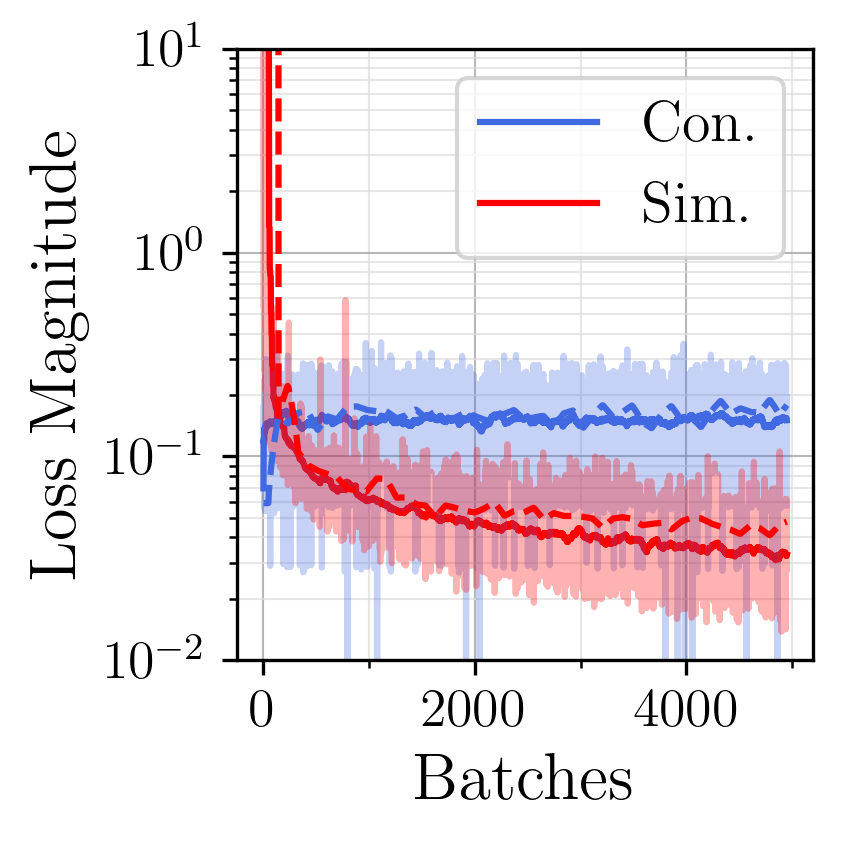}
        \caption{Contrastive vs. Similarity}
        \label{fig:conv_a}
    \end{subfigure}
    \hfill
    \begin{subfigure}[b]{0.24\linewidth}
        \centering
        \includegraphics[width=\linewidth]{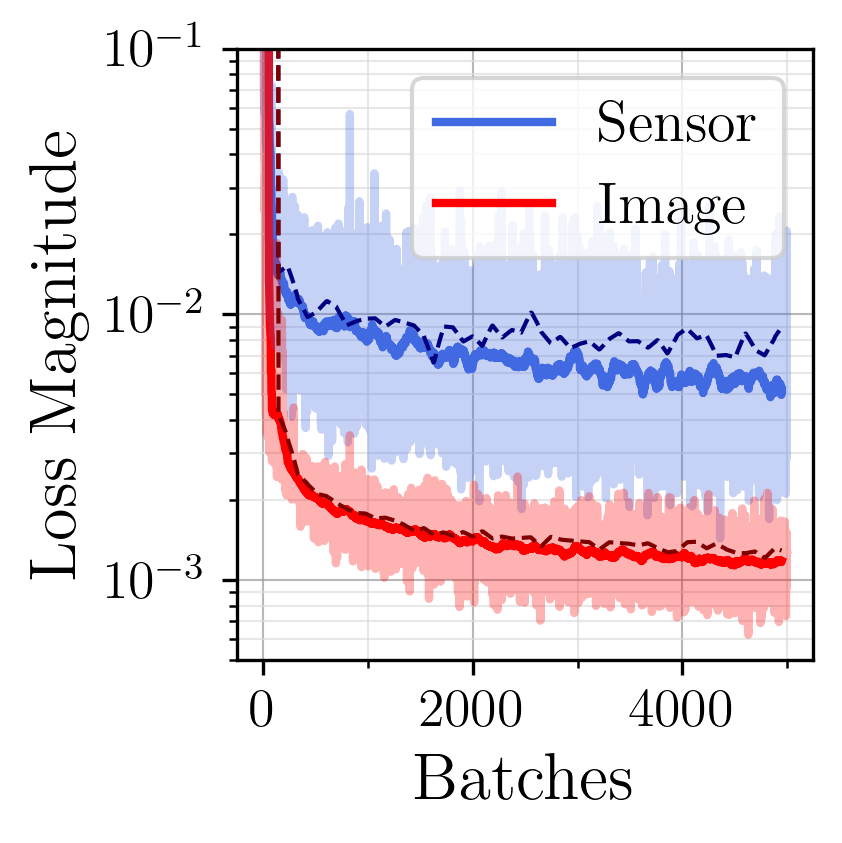}
        \caption{Sensor vs. Image}
        \label{fig:conv_b}
    \end{subfigure}
    \begin{subfigure}[b]{0.24\linewidth}
        \centering
        \includegraphics[width=\linewidth]{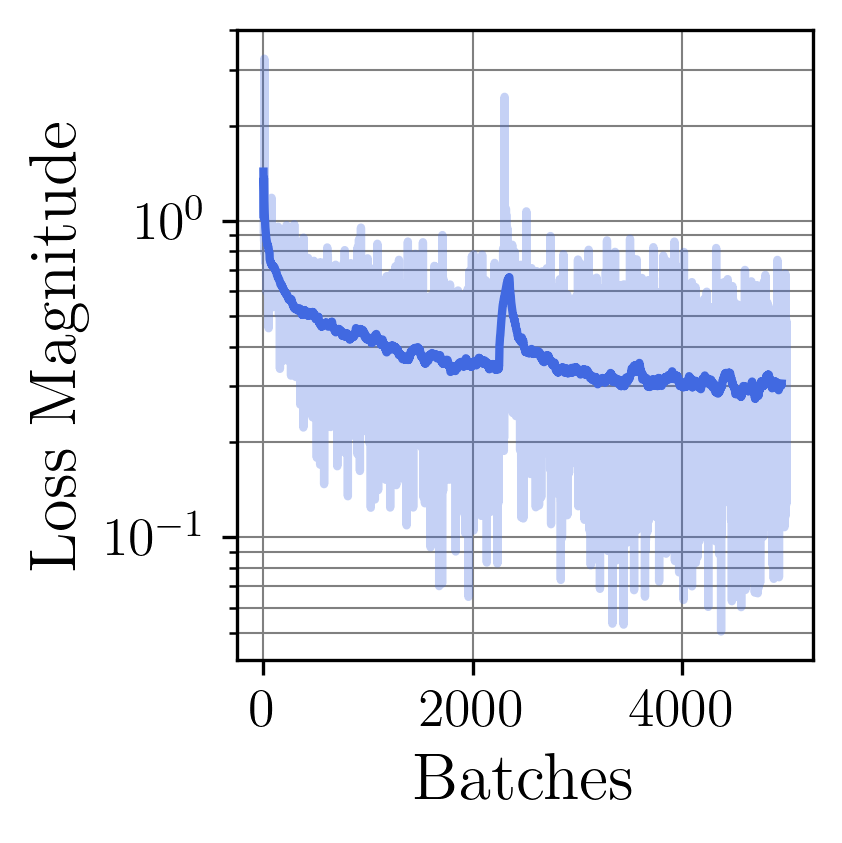}
        \caption{Anomaly Loss}
        \label{fig:conv_c}
    \end{subfigure}
    \hfill
    \begin{subfigure}[b]{0.24\linewidth}
        \centering
        \includegraphics[width=\linewidth]{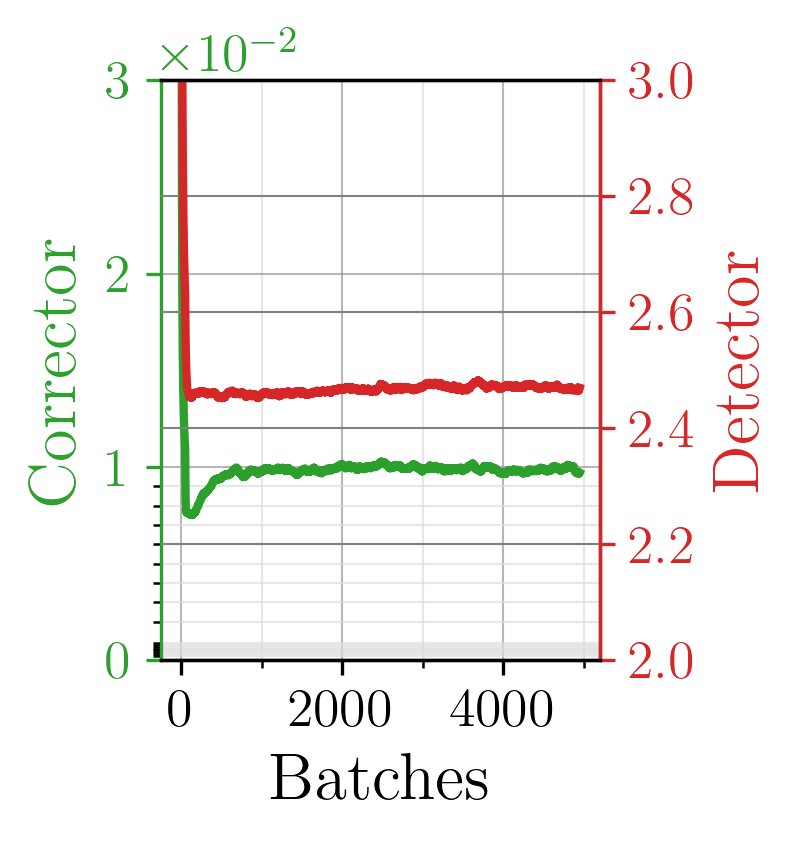}
        \caption{Lipschitz Constraints}
        \label{fig:conv_d}
    \end{subfigure}
    \caption{Training dynamics and spectral properties of MMSSL. (a) Both contrastive and similarity losses converge rapidly, with test losses (dashed) closely tracking training curves. (b) Reconstruction errors for both modalities decrease monotonically, confirming stable learning. (c) The anomaly detection loss stabilizes after an initial adaptation phase. (d) The spectral norm of the Corrector (green) is strictly bounded ($< 0.03$), ensuring a contractive mapping, while the Detector (red) maintains high sensitivity ($> 2.0$).}
    \label{fig:convergence_analysis}
\end{figure}

To verify the optimization stability, we compare the training loss trajectories in \autoref{fig:convergence_analysis}. As illustrated in \autoref{fig:conv_a} and \autoref{fig:conv_b}, our self-supervised objective yields a smooth and fast convergence profile. The contrastive and similarity losses in \autoref{fig:conv_a} decrease monotonically, indicating that the encoder successfully learns a coherent multimodal latent space. Crucially, the test losses (dashed lines) closely track the training curves without divergence, demonstrating robust generalization and the absence of overfitting. Similarly, the reconstruction losses for both sensor and image modalities in \autoref{fig:conv_b} exhibit stable decay, reaching magnitudes of $10^{-3}$ and $10^{-2}$ respectively. Furthermore, we analyze the spectral properties of the network to validate our theoretical stability guarantees. \autoref{fig:conv_d} plots the evolution of the estimated Lipschitz constants for the detection and correction modules. The results confirm our layer-specific regularization strategy. The detector maintains a high Lipschitz value ($L > 2.0$), preserving the high-frequency sensitivity required to identify subtle anomalies. In stark contrast, the correction is effectively constrained to a low-Lipschitz regime ($L \approx 0.01$), empirically satisfying the contraction mapping condition $L < 1$. This spectral gap ensures that while the system remains highly sensitive to faults, the correction process dampens perturbations, preventing error amplification during the recovery phase.

\subsubsection{Failure detection}
\label{ssec:failure_detection}

While robust correction is the primary objective, the reliability of the system hinges on its ability to accurately flag anomalous states. \autoref{tab:fault_detection} summarizes the fault detection performance aggregated over 250 batches.

\begin{table}[htb]
\centering
\caption{Fault detection performance aggregated over 250 batches. Metrics are macro-averaged by class.}
\begin{tabular}{lccc}
\toprule
\textbf{Class} & \textbf{Precision} & \textbf{Recall} & \textbf{F1-Score} \\
\midrule
Clean & 0.929 & 0.999 & 0.963 \\
Sensor Failure & 0.997 & 0.839 & 0.911 \\
Camera Failure & 0.996 & 0.860 & 0.923\\
\midrule
\textbf{Macro Avg} & \textbf{0.974} & \textbf{0.900} & \textbf{0.933} \\
\bottomrule
\end{tabular}
\label{tab:fault_detection}
\end{table}

The proposed framework achieves a high macro F1-Score of 0.933, driven by near-perfect precision across all classes ($>99\%$ for faults). This indicates that the system is highly conservative. It generates virtually no false alarms, a critical property for autonomous industrial monitors where false positives can lead to costly downtime. To understand the optimization stability, we analyze the metric evolution in \autoref{fig:training_dynamics}. 

\begin{figure}[htb]
     \centering
     \begin{subfigure}[t]{0.24\linewidth}
         \centering
         \includegraphics[width=\linewidth]{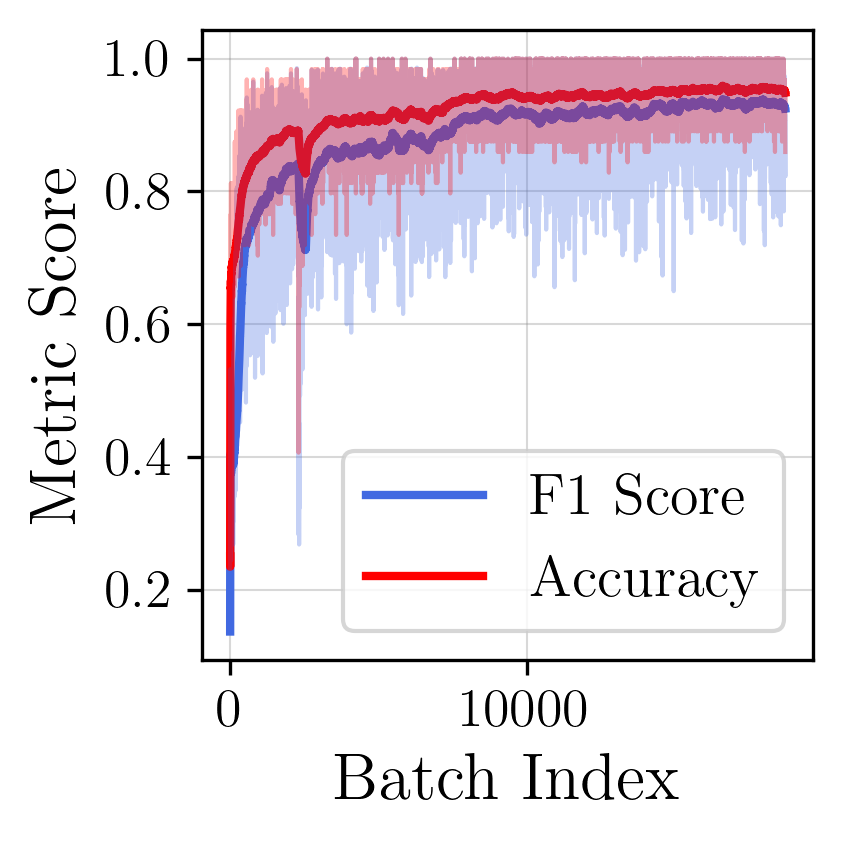}
         \caption{F1 \& Accuracy}
         \label{sfig:f1}
     \end{subfigure}
     \begin{subfigure}[t]{0.24\linewidth}
         \centering
         \includegraphics[width=\linewidth]{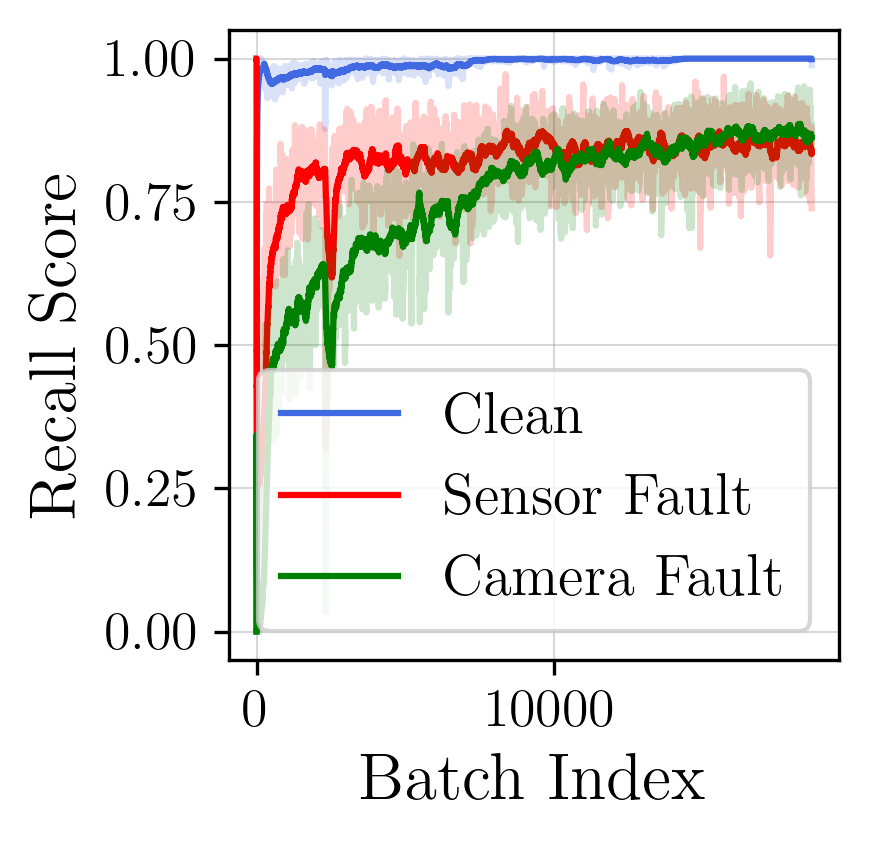}
         \caption{Class-wise Recall}
         \label{sfig:dynamic}
     \end{subfigure}
    \caption{Training dynamics of the failure detector. (a) The model converges rapidly to a stable high-performance regime. (b) While clean states are identified immediately, the model progressively improves its sensitivity to sensor and camera faults, eventually plateauing at robust recall levels.}
    \label{fig:training_dynamics}
\end{figure}

The global accuracy and F1-score in \autoref{sfig:f1} exhibit rapid convergence, stabilizing within the first 3,000 batches with decreasing variance. \autoref{sfig:dynamic} decompose this performance by class and reveals the expected hierarchy of difficulty. The detection of the clean state (blue) saturates almost immediately, whereas the fault classes require a longer adaptation phase. Notably, the Sensor Fault recall (red) and Camera Fault recall (green) rise competitively, eventually stabilizing at $\approx 85\%$. The initial volatility in the fault curves highlights the challenge of learning the separation margin $m$ in the latent space during the early contrastive phase. The aggregated confusion matrix in \autoref{fig:confusion_matrix} provides insight into the remaining error modes. 

\begin{figure}[htb]
     \centering
        \includegraphics[width=.5\linewidth]{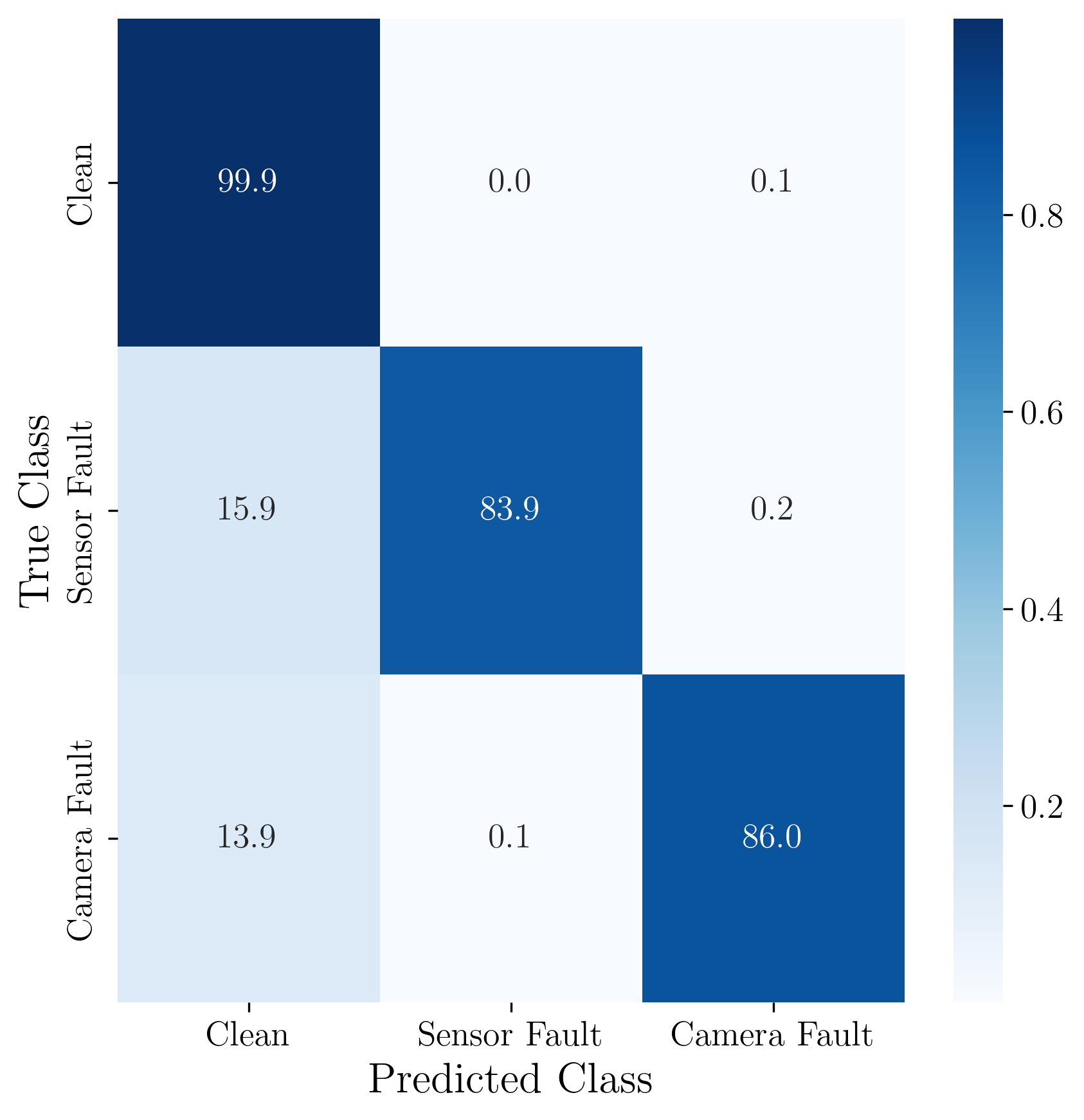}
    \caption{Normalized confusion matrix aggregated over 250 batches. The dominant diagonal confirms high classification accuracy. Off-diagonal errors are almost exclusively false negatives, identifying the detection threshold for subtle anomalies as the limiting factor.}
    \label{fig:confusion_matrix}
\end{figure}

The model exhibits perfect cross-modal separation. The confusion between sensor faults and camera faults is negligible ($<0.2\%$). This confirms that the encoder successfully maps these modalities to distinct subspaces as designed. The primary error source is the misclassification of faults as clean (False Negatives: $15.9\%$ for sensor, $13.9\%$ for camera). This suggests that certain low-magnitude perturbations fall within the natural variance of the clean data manifold or the enforced margin $m$. However, given the high precision ($>99\%$), when the model \textit{does} signal a fault, the prediction is exceptionally reliable.

\subsubsection{Failure Propagation Analysis}
\label{ssec:propagation_analysis}

To empirically validate the theoretical analysis regarding perturbation propagation derived in \autoref{sec:pertub}, we analyzed the latent space activation patterns under localized input faults. \mylemmaref{lem:dense_add} and \mylemmaref{lem:conv_add} suggest a fundamental structural difference: dense layers are expected to diffuse localized faults globally across the feature space, whereas convolutional layers should preserve spatial locality, resulting in a sparse error profile. To verify this, we injected localized block perturbations into the input and tracked the magnitude of the latent deviation $|\delta_z| = |z_{\mathrm{fault}} - z_{\mathrm{clean}}|$ across feature dimensions. \autoref{fig:heatmaps} visualizes the raw error magnitude across 150 test samples. 

\begin{figure}[htb]
     \centering
     \begin{subfigure}[b]{.48\linewidth}
         \centering
         \includegraphics[width=\linewidth]{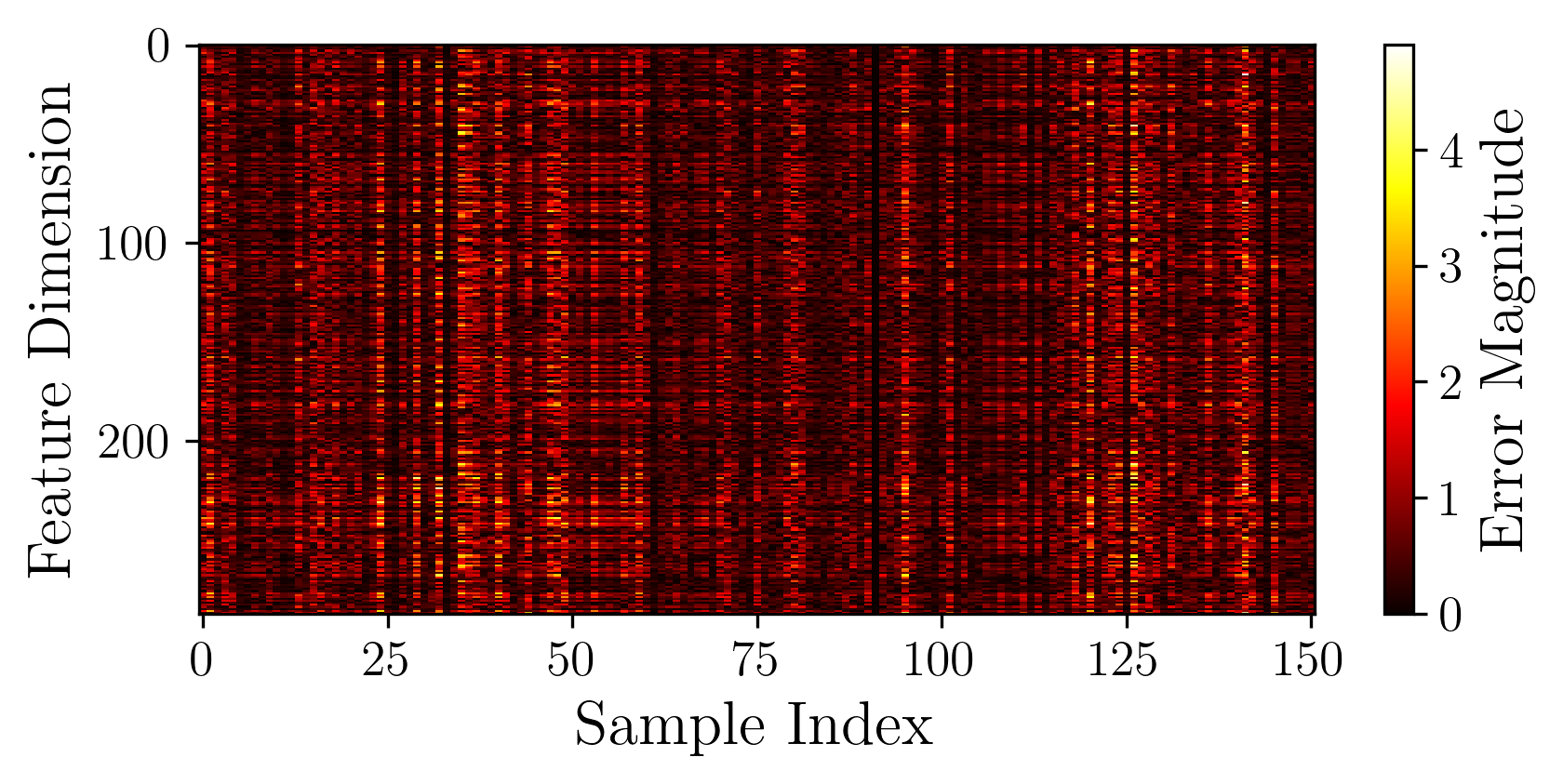}
         \caption{Fully connected}
     \end{subfigure}
     \begin{subfigure}[b]{.48\linewidth}
         \centering
         \includegraphics[width=\linewidth]{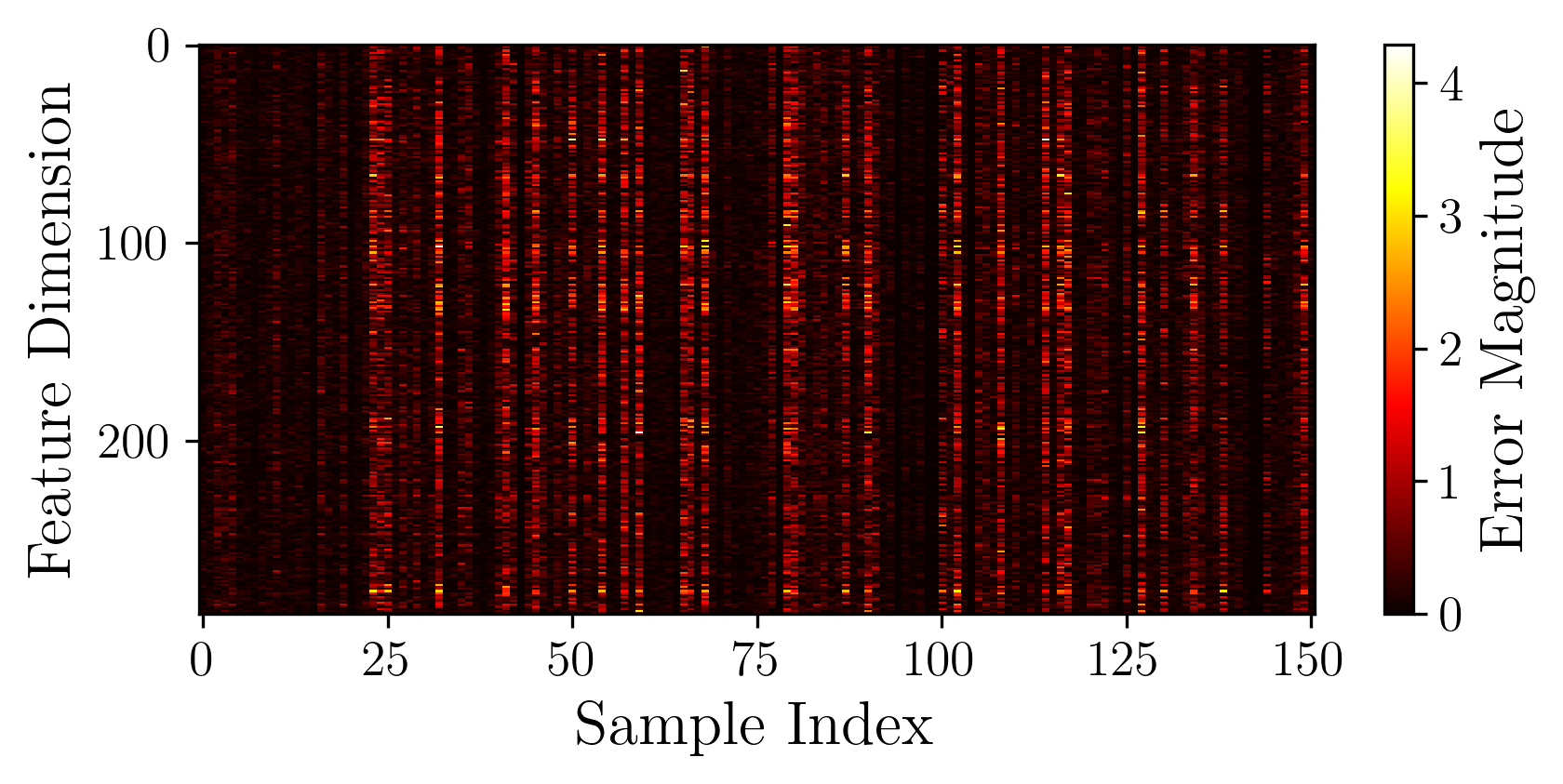}
         \caption{Convolutional}
     \end{subfigure}
    \caption{Latent error heatmaps for dense (a) and convolutional (b) encoders.}
    \label{fig:heatmaps}
\end{figure}

The dense sensor encoder exhibits a blur effect, characterized by continuous vertical lines where a single input fault activates nearly all latent dimensions simultaneously. This confirms the global mixing property of fully connected layers. In contrast, the convolutional camera encoder displays a distinct block-sparse pattern. Active error regions are separated by large zones of zero magnitude (black), indicating that the perturbation remains spatially confined to specific feature maps, preserving the geometric structure of the anomaly. To quantify this behavior, we computed the distribution of absolute error magnitudes in the latent space, shown in \autoref{fig:error_dist}. 

\begin{figure}[htb]
    \centering
    \includegraphics[width=.5\linewidth]{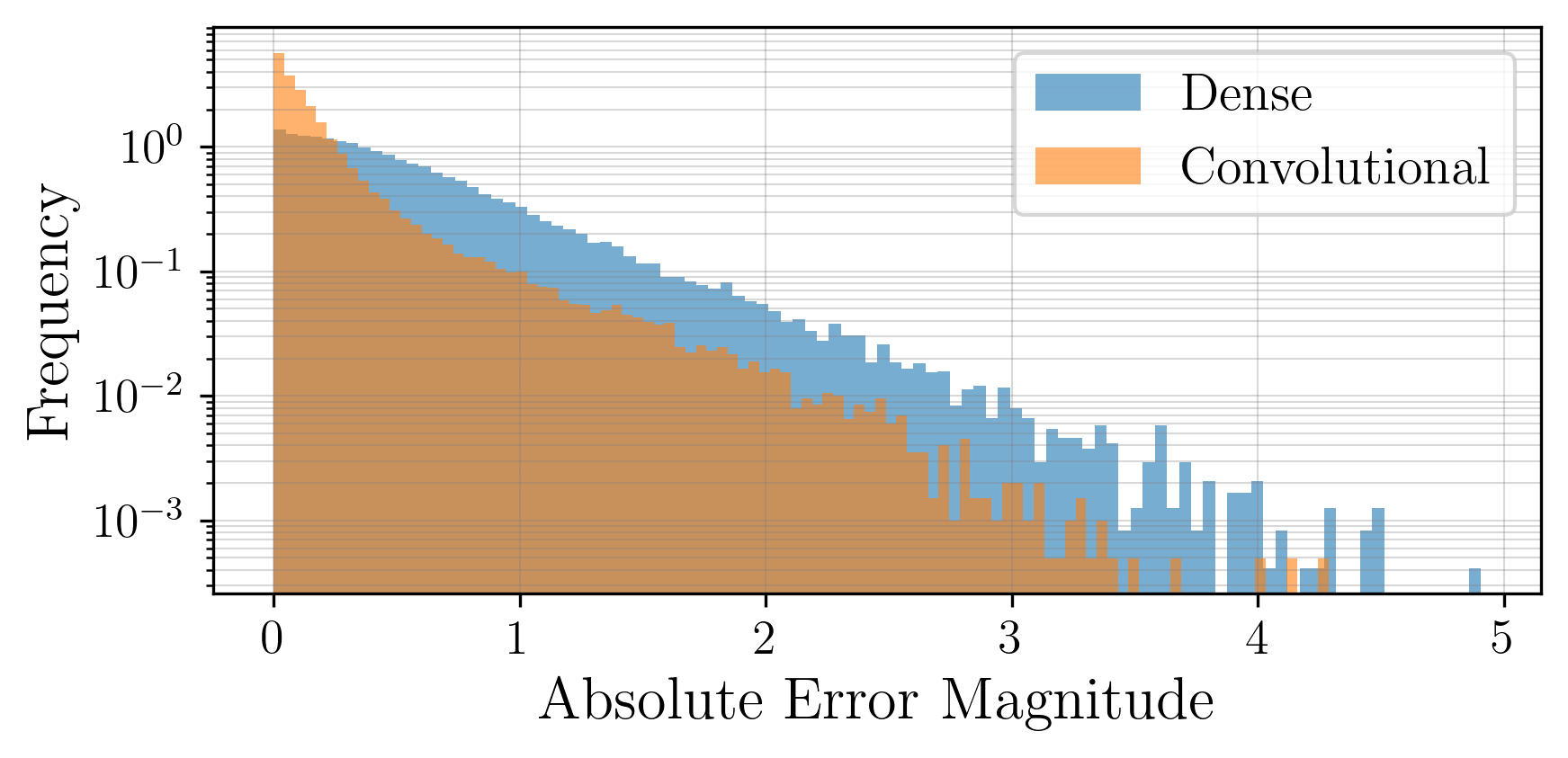}
    \caption{Log-histogram of latent error magnitudes. The convolutional encoder (orange) exhibits a sharp peak at zero, whereas the dense encoder (blue) shows a broad distribution, confirming global diffusion.}
    \label{fig:error_dist}
\end{figure}

The log-scale histogram reveals a decisive difference in the error topology. The convolutional distribution (orange) is heavily zero-inflated, featuring a dominant peak at $|\delta_z| \approx 0$ and a heavy tail. This sparsity implies that for any given localized fault, the majority of latent features remain uncorrupted, maximizing the signal-to-noise ratio for the subsequent detector. Conversely, the dense distribution (blue) lacks this zero-peak and instead follows a broad, bell-shaped distribution shifted away from zero. This confirms that the fault energy is distributed entropically across the entire latent vector, diluting the local anomaly signal. Collectively, these results empirically confirm \mylemmaref{lem:dense_add} and \mylemmaref{lem:conv_add}.

\subsection{Ablation studies}

To examine the contribution of individual components within the MMSSL framework, we conducted a series of ablation studies. Specifically, we isolate the impact of the proposed layer-wise Lipschitz regularization strategy on reconstruction stability, and we analyze the sensitivity of the model to the weighting of the self-supervised loss terms.

First, we investigate the spectral constraints imposed on the compute block. \autoref{tab:ablation} compares the performance of the full MMSSL model against the baseline (w/o Reg) where the spectral normalization, gradient clipping, and amplification mechanisms described in \autoref{sec:method} are removed.

\begin{table}[htb]
\centering
\caption{Train and test performances (in $10^{-2}$) of the proposed SSL approach with and without regularizer. Values represent the mean over the final 250 batches of 10 trials for camera, sensor, and both combined correction performance.}
\label{tab:ablation}
\begin{tabular}{l c c c c }
\toprule
\textbf{Modality} & \textbf{w/o Reg} & \textbf{MMSSL} \\
\midrule
& \multicolumn{2}{c}{\textbf{Training}} \\
\midrule
Camera   & 0.127 $\pm$ 0.024 & \bestvalue{0.098 $\pm$ 0.023} \\
Sensor   & 0.362 $\pm$ 0.349 & \bestvalue{0.307 $\pm$ 0.302} \\
Combined   & 0.490 $\pm$ 0.357 & \bestvalue{0.405 $\pm$ 0.310} \\
\midrule
& \multicolumn{2}{c}{\textbf{Testing}}\\ 
\midrule
Camera   & 0.147 $\pm$ 0.003 & \bestvalue{0.118 $\pm$ 0.005} \\
Sensor   & 0.499 $\pm$ 0.047 & \bestvalue{0.405 $\pm$ 0.040} \\
Combined   & 0.645 $\pm$ 0.046 & \bestvalue{0.523 $\pm$ 0.042} \\
\bottomrule
\end{tabular}
\end{table}

The unregularized model consistently converges to a suboptimal plateau, yielding a significantly higher combined test error ($0.645$) compared to MMSSL ($0.523$). This $\approx 19\%$ reduction in error demonstrates that the Lipschitz constraints do more than just stabilize training. They actively shape the optimization landscape. By enforcing a contractive mapping, the regularizer guides the network away from easy-to-reach but poor local minima and towards a more physically accurate correction manifold. This effect is consistent across modalities, with the camera error dropping from $0.147$ to $0.118$ and the sensor error from $0.499$ to $0.405$, proving that the geometric constraints are essential for maximizing reconstruction fidelity.

To quantify the trade-off between the two self-supervised objectives, \autoref{fig:ablation_tradeoff} presents the impact of the weighting ratio $\lambda_{sim} / \lambda_{con}$ on system performance. The dual-axis plot visualizes the detection accuracy (orange, right axis) against the correction loss (blue, left axis).

\begin{figure}[htb]
    \centering
    \includegraphics[width=.5\linewidth]{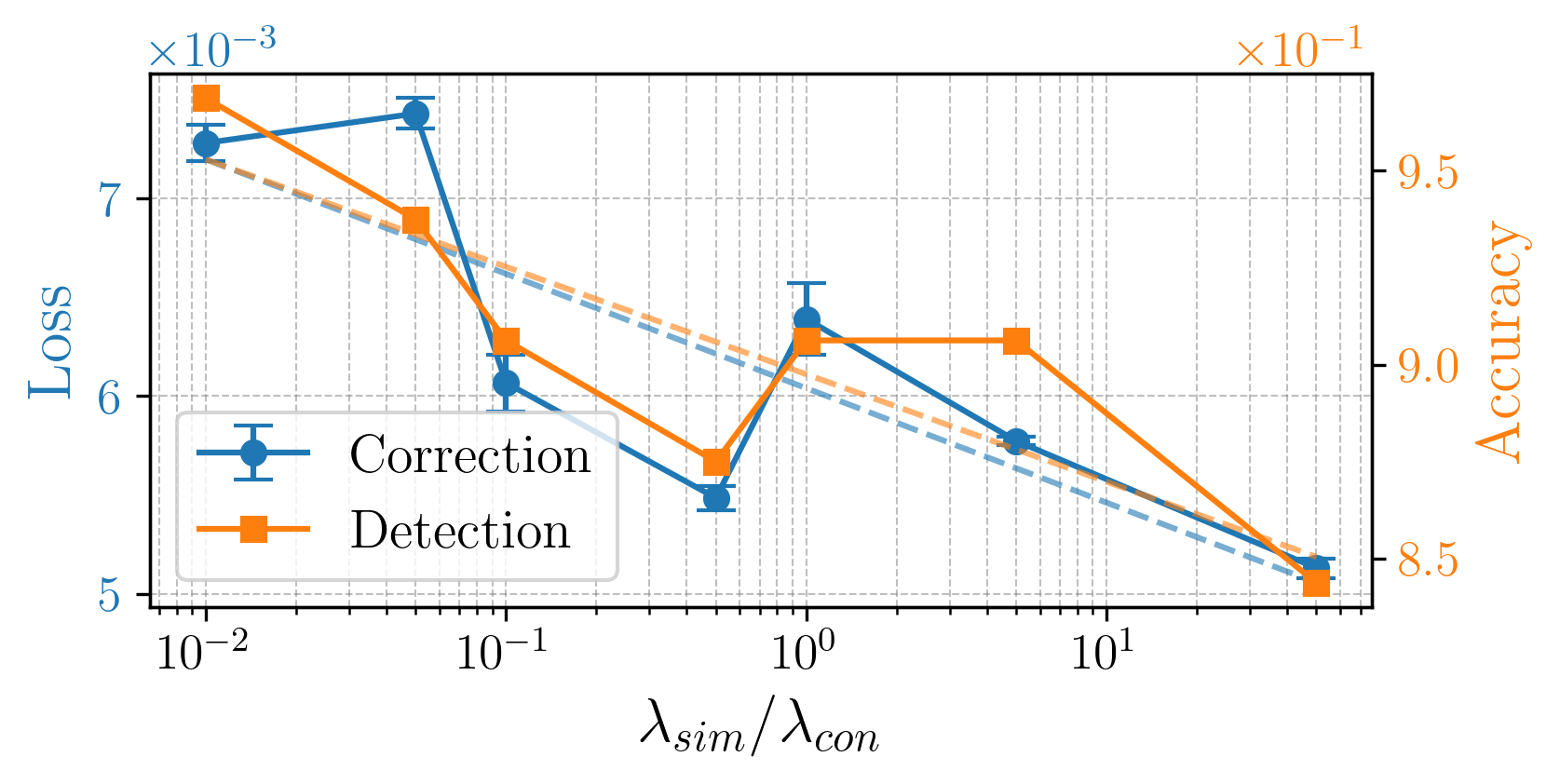}
    \caption{Dual-axis plot showing the Correction Loss (left y-axis, blue, scale $\times 10^{-3}$) and Detection Accuracy (right y-axis, orange, scale $\times 10^{-1}$) as a function of the $\lambda_{sim}/\lambda_{con}$ ratio (x-axis, logarithmic scale). The blue line with error bars represents the correction loss, and the orange line with square markers represents the detection accuracy.}
    \label{fig:ablation_tradeoff}
\end{figure}

The results reveal a clear inverse relationship between the competing regularizers. In the contrastive-dominant regime ($\lambda_{sim} \ll \lambda_{con}$), detection accuracy peaks at $98\%$ as the model enforces strong latent separation, but this distorts the manifold geometry, resulting in higher correction errors. Conversely, as the ratio increases to favor similarity ($\lambda_{sim} \gg \lambda_{con}$), the correction loss drops to its minimum ($5.1 \times 10^{-3}$) due to stronger contraction, but this smoothing collapses the anomaly boundary, degrading detection accuracy to $84.5\%$. A balanced ratio ($\approx 10^0$) offers the optimal compromise, maintaining high discriminative power ($>91\%$) while stabilizing the reconstruction.

To visualize the internal optimization dynamics of the self-supervised framework, \autoref{fig:sim_con_ablation} decomposes the total objective into its constituent parts. The dual-axis plot tracks the raw, unweighted values of the similarity loss (blue, left axis) and the contrastive loss (red, right axis) as a function of their weighting ratio $\lambda_{sim} / \lambda_{con}$.

\begin{figure}[htb]
    \centering
    \includegraphics[width=.5\linewidth]{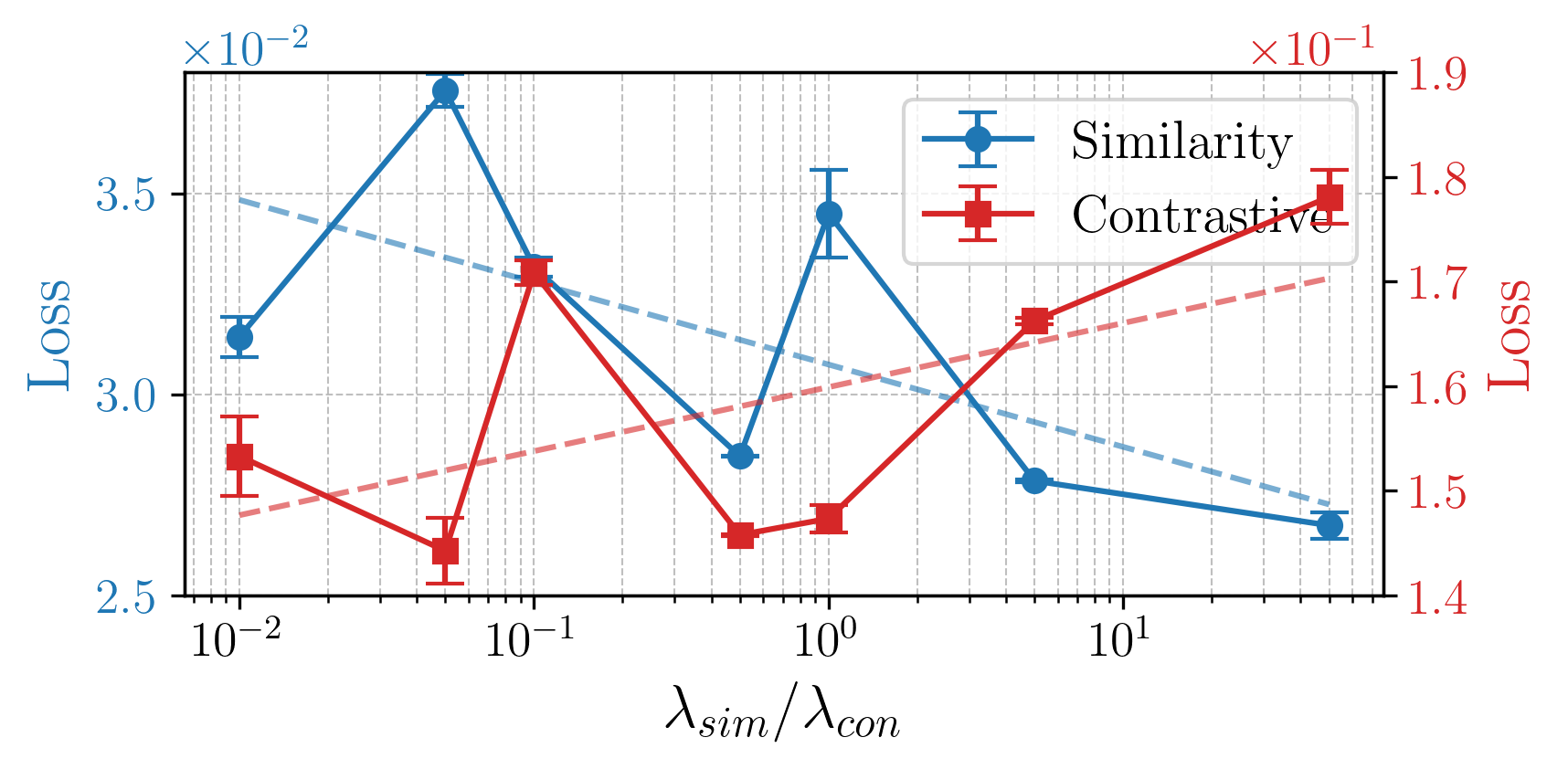}
    \caption{Pareto trade-off between the unweighted self-supervised loss components. As the weighting ratio $\lambda_{sim} / \lambda_{con}$ increases, the similarity loss (blue) decreases, indicating better reconstruction, while the contrastive loss (red) rises, indicating reduced latent separability. The dashed lines illustrate the opposing linear trends of these competing objectives.}
    \label{fig:sim_con_ablation}
\end{figure}

The empirical results demonstrate a fundamental opposition between the two geometric objectives. As the ratio increases (moving right), the system prioritizes similarity, driving the unweighted similarity loss down from a peak of $\approx 3.8 \times 10^{-2}$ to a minimum of $\approx 2.7 \times 10^{-2}$. This confirms that stronger regularization effectively forces the latent manifold to contract. However, this contraction comes at a direct cost to separability: the unweighted contrastive Loss rises steadily from $\approx 1.5$ to $\approx 1.8 \times 10^{-1}$ (red trend line), indicating that the embeddings are becoming less distinguishable. The intersection of the trend lines near $\lambda_{sim} / \lambda_{con} \approx 10^0$ highlights a critical equilibrium point where neither objective dominates, allowing the model to minimize reconstruction error without completely collapsing the latent boundaries required for detection.

To specifically assess the robustness of the anomaly detector under varying regularization regimes, \autoref{fig:acc_f1_ablation} details the evolution of classification metrics. The dual-axis plot displays the overall accuracy (blue, left axis) and the class-balanced F1 score (green, right axis) as a function of the loss weighting ratio $\lambda_{sim} / \lambda_{con}$.

\begin{figure}[htb]
    \centering
    \includegraphics[width=.5\linewidth]{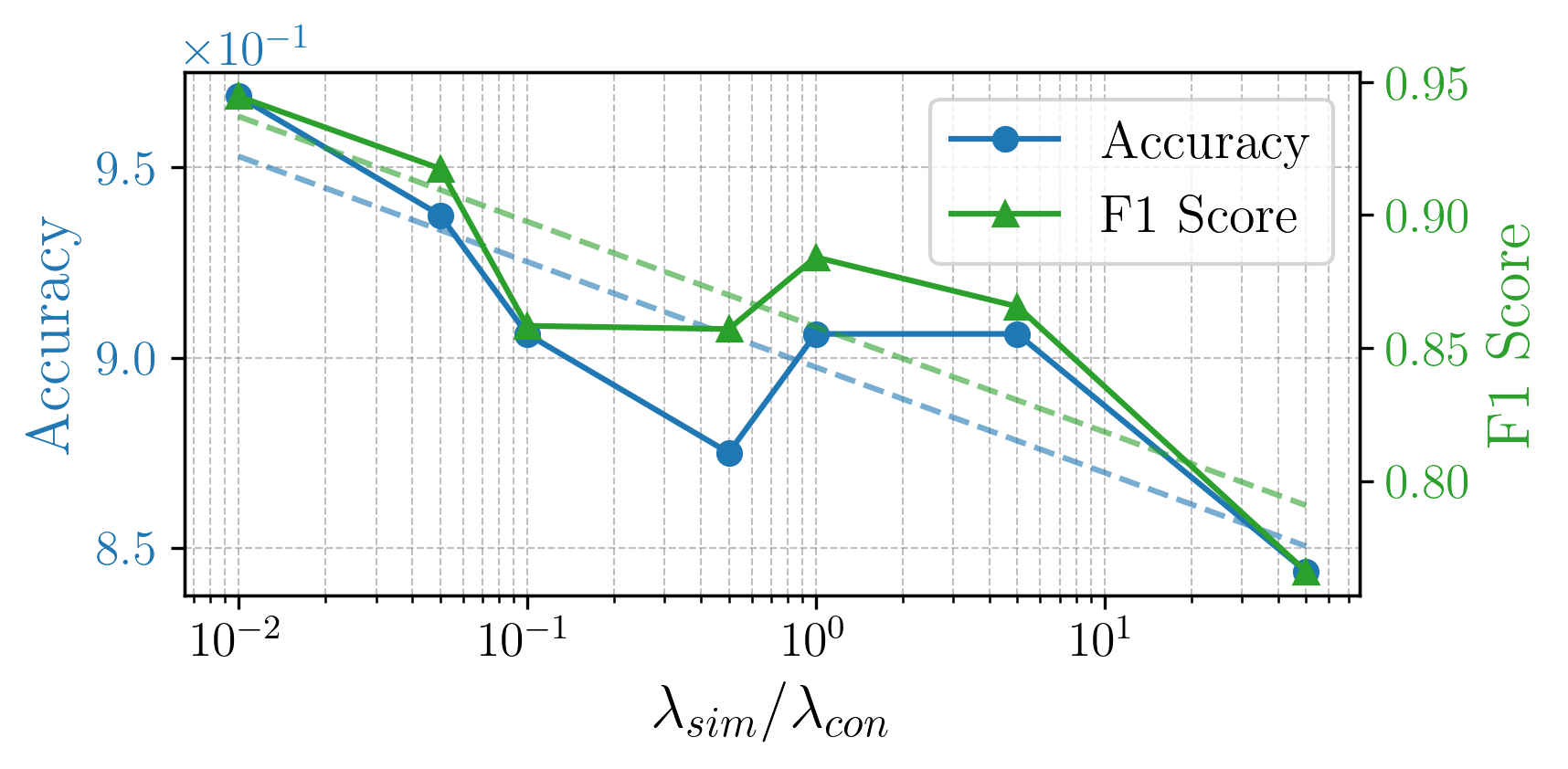}
    \caption{Dependence of detection performance on the loss weighting ratio. Both accuracy (blue) and F1 score (green) degrade as the similarity weight increases, confirming that strong contrastive regularization is essential for maintaining high discriminative power.}
    \label{fig:acc_f1_ablation}
\end{figure}

The results exhibit a consistent downward trend that mirrors the contrastive loss behavior observed in the previous analyses. Performance is maximized in the contrastive-dominant regime ($\lambda_{sim} = 10^{-2}$), where the F1 score reaches nearly $0.95$ and accuracy peaks at $97\%$. This confirms that a strong margin constraint is the primary driver for separability. As the system shifts towards a balanced weighting ($\lambda_{sim} / \lambda_{con} \approx 10^0$), there is a slight dip, yet the model maintains robust performance with an F1 score of $\approx 0.88$ and accuracy $>90\%$. This stability plateau at unity is critical, as it allows the system to support accurate failure correction (as seen in the previous figures) without catastrophic degradation in detection. However, as the similarity term dominates ($\lambda_{sim} \ge 10$), the F1 score plummets below $0.80$, indicating that the detector can no longer reliably distinguish between the heavily smoothed anomaly representations and the clean manifold.


\section{Conclusion}
\label{sec:conclusion}

This paper presented a theoretically motivated approach for robust multimodal learning under partial or corrupted input conditions. By linking perturbation propagation in convolutional and dense layers to their expected Lipschitz behavior, we established a formal basis for selecting architectural components that either localize or diffuse input anomalies. Building on these findings, we developed a two-stage self-supervised model consisting of a pre-trained multimodal autoencoder and a corrective compute block operating in latent space. The method integrates contrastive alignment for anomaly identification with linear reconstruction for correction, while employing selective Lipschitz regularization and gradient clipping to control sensitivity at different depths of the network. The results demonstrate that balancing sensitivity and contraction across layers yields architectures that can both detect and correct faults, a key requirement for reliable multimodal systems in industrial and autonomous settings.

Future research will focus on extending this framework in three directions. First, integrating adaptive Lipschitz control through dynamic spectral normalization could enable online sensitivity modulation during deployment. Second, incorporating probabilistic uncertainty estimation into the correction block may improve interpretability and safety.

\bibliographystyle{unsrt}
\bibliography{references.bib}
\end{document}